\documentclass[a4paper,11pt]{article}

\usepackage[utf8]{inputenc} 
\usepackage[T1]{fontenc}    

\usepackage{graphicx}
\usepackage{subfig}
\usepackage{booktabs} 
\usepackage{nicefrac}       
\usepackage[margin=1in]{geometry}
\usepackage{url}
\usepackage{smile}
\usepackage{algorithm}
\usepackage{algorithmic}
\usepackage{todonotes}
\usepackage{epstopdf}
\usepackage{wrapfig}
\usepackage{enumitem}
\usepackage{mathtools}

\usepackage[colorlinks, linkcolor=blue, anchorcolor=blue, citecolor=blue]{hyperref}       

\usepackage{kpfonts}
\DeclareMathAlphabet{\mathsf}{OT1}{cmss}{m}{n}

\SetMathAlphabet{\mathsf}{bold}{OT1}{cmss}{bx}{n}

\usepackage{natbib}

\usepackage{pifont}

\usepackage{makecell}

\newcommand{\mone}{\mathds{1}}

\newcommand{\mtoned}{\widetilde{\mathds{1}}_{\Delta}}

\newcommand{\ttimes}{\widetilde{\times}}

\newcommand{\supp}{\mathrm{supp}}
\newcommand{\diam}{\mathrm{diam}}

\newcommand{\ReLU}{\mathrm{ReLU}}
\newcommand{\Conv}{\mathrm{Conv}}
\newcommand{\id}{\mathrm{id}}

\newcommand{\convex}{\mathrm{ch}}

\newcommand{\vhalf}{\vspace{0cm}}
\newcommand{\vfull}{\vspace{0cm}}

\newtheorem*{theorem*}{Theorem}

\begin{document}


\title{Benefits of Overparameterized Convolutional Residual Networks: Function Approximation under Smoothness Constraint}
\author{Hao Liu, Minshuo Chen, Siawpeng Er, Wenjing Liao, Tong Zhang and Tuo Zhao \thanks{Hao Liu is affiliated with the Department of Mathematics at Hong Kong Baptist University; Minshuo Chen, Siawpeng Er and Tuo Zhao are affiliated with the ISYE department at Georgia Tech; Wenjing Liao is affiliated with the School of Mathematics at Georgia Tech; Tong Zhang is affiliated with the Department of Mathematics  and Department of Computer Science and Engineering at The Hong Kong University of Science and Technology, and Google Research; Email: \text{haoliu@hkbu.edu.hk, $\{$mchen393, ser8, wliao60, tourzhao$\}$@gatech.edu}, tongzhang@tongzhang-ml.org.}}

\date{}

\maketitle

\begin{abstract}
Overparameterized neural networks enjoy great representation power on complex data, and more importantly yield sufficiently smooth output, which is crucial to their generalization and robustness. Most existing function approximation theories suggest that with sufficiently many parameters, neural networks can well approximate certain classes of functions in terms of the function value. The neural network themselves, however, can be highly nonsmooth. To bridge this gap, we take convolutional residual networks (ConvResNets) as an example, and prove that large ConvResNets can not only approximate a target function in terms of function value, but also exhibit sufficient first-order smoothness. Moreover, we extend our theory to approximating functions supported on a low-dimensional manifold. Our theory partially justifies the benefits of using deep and wide networks in practice. Numerical experiments on adversarial robust image classification are provided to support our theory.
\end{abstract}

\vfull
\section{Introduction}\label{sec:intro}
\vfull
Deep neural networks of enormous sizes have achieved remarkable success in various applications. Some well-known examples include ViT-Huge of $632$ million parameters \citep{dosovitskiy2020image}, BERT-Large of $336$ million parameters \citep{devlin2018bert}, and the gigantic GPT-3 of $175$ billion parameters \citep{brown2020language}. In addition to outstanding testing accuracy, there has been evidence that large neural networks favor smoothness and yield good robustness \citep{madry2017towards, bubeck2021universal}.

Among vast literature on explaining the success of neural networks, universal approximation theories analyze how well neural networks can represent complex data models (see literature in related work section). These works focus on approximating a target function in terms of its function value (i.e., in function $L_\infty$ norm). However, other important properties, espcifically the smoothness of the neural networks, are less investigated. A few early results provide asymptotic results on two-layer networks with smooth activation for approximating both function value and derivatives \citep{hornik1990universal, cardaliaguet1992approximation}. Recently, \citet{guhring2020error,hon2021simultaneous} established nonasymptotic approximation theory of feedforward networks in terms of Sobolev norms.

In real-world applications, on the other hand, practitioners empirically demonstrated a close tie between the smoothness of a trained neural network to its adversarial robustness \citep{gu2014towards, hein2017formal, weng2018evaluating, miyato2018virtual}. The intuition behind is relatively clear. Consider, for instance, adding some adversarial perturbation to an input. A network of small (local) Lipschitz constant produces less deviation to the original output, and therefore, is often resilient to adversarial attackes. On the contrary, a network that is vulnerable to adversarial attacks usually has a large Lipschitz constant. Over the years, many computational methods are proposed and extensively tested in experiments for promoting network smoothness \citep{goodfellow2014explaining, madry2017towards, miyato2018virtual,zhang2019theoretically}. Apart from these explicit training methodologies, the size of a network is also recognized as a critical factor to its generalization and robustness \citep{zagoruyko2016wide, madry2017towards, wu2020wider}. Yet, theoretical understanding is largely missing.

In this paper, we investigate universal approximation ability of neural networks with smoothness guarantees. We consider the convolutional residual networks (ConvResNet, see a description in Section \ref{sec:convresnet}) with ReLU activation as an example. We measure the approximation error of ConvResNet in terms of not only the function value, but also higher order smoothness. Specifically, suppose given a target function $f$ belonging to a Sobolev space in a $D$-dimensional hypercube. We provide an approximation error estimate in terms of Sobolev norm as a function of the size of ConvResNet. We also extend our theory to functions supported on a $d$-dimensional Riemannian manifold ($d \ll D$). We summarize our main results in the following informal theorem.

\vhalf
\begin{theorem}[informal]\label{thm.informal}
Consider a ConvResNet architecture with $\tilde{M}$ residual blocks and each convolutional filter having at most $\tilde{J}$ channels. Let $\alpha \geq 2$ and $1\leq p \leq \infty$ be positive integers. Then

$\bullet$ (Euclidean) for any target function in a Sobolev space $W^{\alpha, p}((0, 1)^D)$ with Sobolev norm $\norm{f}_{W^{\alpha, p}((0, 1)^D)} \leq 1$, there exists $\tilde{f}$ yielded by the ConvResNet architecture, such that
$$
\norm{\tilde{f} - f}_{W^{s, p}} \leq {\sf const} \cdot (\tilde{M}\tilde{J})^{-\frac{\alpha-s}{D}} \quad \text{for} \quad s \in [0, 1]
$$
with the constant depending on $D, \alpha$, $p$;

$\bullet$ (Manifold) given $\cM \subset \RR^D$ a $d$-dimensional Riemannian manifold satisfying mild regularity conditions, for any target function in a Sobolev space $W^{\alpha, \infty}(\cM)$ with $\norm{f}_{W^{\alpha, \infty}(\cM)} \leq 1$, there exists $\tilde{f}$ yielded by the ConvResNet architecture, such that
$$
\norm{\tilde{f} - f}_{W^{k, \infty}} \leq {\sf const} \cdot (\tilde{M}\tilde{J})^{-\frac{\alpha-k}{d}} \quad \text{for} \quad k \in \{0, 1\}
$$
with the constant depending on $\alpha, p, \cM$.
\end{theorem}
Our theory restricts to $s \leq 1$, since only first-order weak derivatives exist for ReLU networks.
Moreover, setting $s = 0$ or $s = 1$ is of particular interest, as $s = 0$ recovers the function value approximation guarantee and $s = 1$ extends the guarantee to first-order derivatives. As can be seen, to achieve the same function value approximation error, $s = 1$ requires a larger network, but enjoys good smoothness. This can partially explain that larger networks are often more robust. We refer readers to Corollary \ref{coro.D} for more discussion.

Theorem \ref{thm.informal} implies that as the number of residual blocks increases or each filter having more channels, ConvResNet gives better approximation of the target function. In order to achieve an $\epsilon$-error, we may set $\tilde{M}\tilde{J} = O(\epsilon^{-\frac{D}{\alpha - s}})$ ($O(\epsilon^{-\frac{d}{\alpha - s}})$ for the manifold case), while there is no scaling restriction between $\tilde{M}$ and $\tilde{J}$. See an explicit configuration of ConvResNet architecture depending on $\tilde{M}$ and $\tilde{J}$ in Theorem \ref{thm.D} and Theorem \ref{thm.M}. (Although the rate in the manifold case is independent of $D$, the network size inevitably weakly depends on $D$.)

Our result on Euclidean spaces is related to \citet{guhring2020error,hon2021simultaneous}, nonetheless, they focus on approximation guarantees of feedforward networks in terms of $W^{s,p}$ norm. It is also worth mentioning that our results are complementary to \citet{bubeck2021universal}, which provides a lower bound on network Lipschitz continuity. \citet{bubeck2021universal} suggest that small network suffers from bad Lipschitz continuity, in fitting isoperimetric random data. However, whether large network enjoys good smoothness is questionable. Our result proves that large network indeed yields appealing Lipschitz continuity from a function approximation perspective.

The manifold case draws motivation from the fact that data in real applications are often governed by a small number of free parameters \citep{tenenbaum2000global, roweis2000nonlinear,coifman2005geometric,allard2012multi}. As a concrete example, \citet{pope2021intrinsic} estimate the intrinsic dimension of many benchmark data sets, including MNIST, CIFAR-10/100, and ImageNet. A striking finding is that the intrinsic dimension of ImageNet is merely around $43$, in a sharp contrast to its $224 \times 224 \times 3$ total pixels. Therefore, it is reasonable to model data as a low-dimensional Riemannian manifold, and we show ConvResNet can adapt to data geometric structures and does not suffer from the curse of ambient dimensionality.

\vfull

\paragraph{Related work}
Approximation theories of feedforward neural network have been studied for a long time, most of which dedicate to function value approximation. The earliest literature dates back to late 1980s. For example, \citet{irie1988capabilities, funahashi1989approximate, cybenko1989approximation, hornik1991approximation, chui1992approximation, leshno1993multilayer} investigated the approximation power of two-layer feedforward neural networks with sigmoidal activation for square integrable functions and established some asymptotic results, where the number of neurons goes to infinity. \citet{barron1993universal, mhaskar1996neural} established nonasymptotic results for the so-called ``Barron'' function space. For multi-layer feedforward neural networks with ReLU activation, \citet{yarotsky2017error} analyzed the approximation of Sobolev $W^{\alpha, \infty}$ functions in a $D$-dimensional hypercube, and proved nonasymptotic results that given a pre-specified approximation error $\epsilon$, the depth and width of neural networks need to be at most of the order $O(\epsilon^{-D/\alpha})$ and $O(\log(1/\epsilon))$, respectively. 
More recently, \citet{suzuki2018adaptivity, suzuki2019deep, liu2021besov} extended to more general function classes such as Besov spaces.

Approximation theories for convolutional networks are established by \citet{zhou2020universality,zhou2020theory,petersen2020equivalence}. In \citet{zhou2020universality}, the authors consider CNN with ReLU activation whose width increases linearly from the first layer to the last. They show that such a CNN can approximate functions in Sobolev $W^{\alpha,2}$ space with arbitrary accuracy for integer $\alpha\geq 2+D/2$. To have a better control on the width of the network, the authors of \citet{zhou2020theory} studied downsampled CNNs, and show that the downsampled CNN can approximate Lipschitz ridge functions with an arbitrary accuracy. In \citet{petersen2020equivalence}, the authors show that any approximation bounds of FNN can be achieved by CNNs. The results in \citet{oono2019approximation,liu2021besov} dedicate to convolutional residual networks. In \citet{oono2019approximation}, the authors show that ConvResNets is able to approximate H\"{o}lder functions with an arbitrary accuracy.

Theoretical results on approximating or learning functions on low-dimensional manifold can be found in \citet{shaham2018provable, chui2018deep, schmidt2019deep, chen2019efficient,chen2019nonparametric,chen2020doubly, nakada1907adaptive, cloninger2020relu, shen2019deep, montanelli2020error, liu2021besov,liu2022deep}. These works show that when the target function is defined on or around a low-dimensional manifold, to achieve an approximation error $\epsilon$, the network size mainly depends on the intrinsic dimension and weakly depends on the ambient dimension.

{\bf Notations}: We use lower case letters to denote scalars, bold lower case letters to denote vectors, upper case letters to denote matrices, and calligraphic letters to denote tensors and sets. For $\xb=[x_1,...,x_D]^{\top},\vb=[v_1,...,v_D]^{\top}$, we denote $\xb^{\vb}=x_1^{v_1}\cdots x_D^{v_D}$ (if well-defined) and $|\vb|=\sum_{i=1}^D |v_i|$. Let $\bm{\alpha} = [\alpha_1, ..., \alpha_D]^\top \in \NN^D$ be a multi-index and $f$ be a function, we denote $D^{\bm{\alpha}} f = \frac{\partial^{|\balpha|} f}{\partial x_{1}^{\alpha_1} \cdots \partial x_D^{\alpha_D}}$. Let $\Omega$ be a subset in $\RR^D$, we denote $\overline{\Omega}$ as its closure and $\convex(\Omega)$ as its convex hull. We use $B_r(\cbb)$ to denote the closed Euclidean ball with radius $r$ and centered at $\cbb$.
\vfull

\section{Preliminary}

\vhalf

\subsection{Sobolev Functions}

\vhalf

We focus on studying neural networks for approximating Sobolev functions. We provide a formal definition of Sobolev functions in both Euclidean spaces and on manifolds. We begin with Sobolev functions in Euclidean spaces \citep[Chapter 8]{brezis2011functional}.

\begin{definition}[Sobolev spaces]
Let $\alpha\geq 0,1\leq p\leq \infty$ be integers, and domain $\Omega\subset\RR^D$. We define Sobolev space $W^{\alpha, p}(\Omega)$ as
\begin{align*}
W^{\alpha,p}(\Omega)=\big\{ &f\in L^p(\Omega): D^{\bm{\alpha}}f\in L^p(\Omega) \mbox{ for all } |\bm{\alpha}|\leq \alpha\big\},
\end{align*}
where $\bm{\alpha}$ is a multi-index.
\end{definition}
For $f\in W^{\alpha,p}(\Omega)$, we define its Sobolev norm as
\begin{align*}
& \|f\|_{W^{\alpha,p}(\Omega)}=\Big(\sum_{|\bm{\alpha}|\leq \alpha} \|D^{\bm{\alpha}}f\|_{L^p(\Omega)}^p\Big)^{1/p}.
\end{align*}
In the special case of $p = \infty$, the Sobolev norm can be rewritten as
$\|f\|_{W^{\alpha,\infty}(\Omega)}=\max_{|\bm{\alpha}|\leq \alpha} \|D^{\bm{\alpha}}f\|_{L^{\infty}(\Omega)}$. In this case, $\|f\|_{W^0,\infty}<\infty$ implies the function value is bounded, and  $\|f\|_{W^1,\infty}<\infty$ implies both the function value and its gradient are bounded.

Our later approximation theories will provide error estimate in terms of Sobolev norms. To allow more flexibility, we define fractional Sobolev norms, which can be viewed as a generalization of Sobolev norms to non-integer $\alpha$. The fractional Sobolev functions are defined as follows.
\begin{definition}[Sobolev--Slobodeckij spaces \citep{slobodeckij1958generalized}]
For $0<s<1$ and $1\leq p\leq \infty$, we define $W^{s,p}(\Omega)$ as
\begin{align*}
W^{s,p}(\Omega)=\left\{ f\in L^p(\Omega): \|f\|_{W^{s,p}(\Omega)}<\infty\right\}
\end{align*}
with
\begin{align*}
& \|f\|_{W^{s,p}(\Omega)}= \Big( \|f\|_{L^p(\Omega)}^p + \int_{\Omega}\int_{\Omega} \Big(\frac{|f(\xb)-f(\yb)|}{\|\xb-\yb\|_2^{s+D/p}}\Big)^p d\xb d\yb\Big)^{1/p} 
\end{align*}
for $1\leq p< \infty$ and
\begin{align*}
& \|f\|_{W^{s,\infty}(\Omega)} =  \max\left\{ \|f\|_{L^{\infty}(\Omega)}, \mbox{ess sup}_{\xb,\yb\in \Omega} \frac{|f(\xb)-f(\yb)|}{\|\xb-\yb\|_2^s} \right\}.
\end{align*}
\end{definition}
We restrict our attention to $s < 1$ for simplicity, as we focus on approximation guarantees up to first-order continuity.

Next, we extend Sobolev spaces to Riemannian manifolds. We provide a brief introduction to manifold; a more detailed description can be found in Appendix \ref{appendix.manifold}. Roughly speaking, a Riemannian manifold $\cM$ is a collection of local neighborhoods, each of which is diffeomorphic to a low-dimensional Euclidean space. These local neighborhoods are termed charts, and a collection of which is an atlas. We provide a formal definition.


\begin{definition}[Atlas]
A smooth atlas for a $d$-dimensional manifold $\cM \subset \RR^D$ is a collection of charts $\{(U_\alpha, \phi_\alpha)\}_{\alpha \in \cA}$, which verifies $\bigcup_{\alpha \in \cA} U_\alpha = \cM$ and $\phi_\alpha: U_\alpha \mapsto \RR^d$ being diffeomorphic and pairwise compatible, i.e.,
\begin{align*}
&\phi_\alpha \circ \phi_\beta^{-1} : \phi_\beta(U_\alpha \cap U_\beta) \to \phi_\alpha(U_\alpha \cap U_\beta) \quad \textrm{and} \quad \phi_\beta \circ \phi_\alpha^{-1} : \phi_\alpha(U_\alpha \cap U_\beta) \to \phi_\beta(U_\alpha \cap U_\beta)
\end{align*} 
are both smooth for any $\alpha,\beta\in \cA$. An atlas is called finite if it contains finitely many charts.
\end{definition}
To define Sobolev spaces on a manifold $\cM$, we shall consider function regularity on each chart, as charts are geometrically ``akin'' to a Eulidean space through the chart mapping $\phi_\alpha$. One caveat, however, is that the chart mapping $\phi_\alpha$ can be arbitrarily rescaled, which results in potential unboundedness. We therefore, fix an atlas on $\cM$ to mitigate this issue. We are ready to define Sobolev spaces on a manifold \citep[Definition 48.17]{driver2003analysis}.


\begin{definition}[Sobolev spaces on manifold]\label{def.sobolevM}
Let $\cM$ be a compact Riemannian manifold of dimension $d$. Let $\{(U_i, \phi_{i})\}_{i=1}^{C_{\cM}}$ be a finite atlas on $\cM$ and $\{\rho_i\}_{i=1}^{C_{\cM}}$ be a partition of unity on $\cM$ such that $\supp(\rho_i)\subset U_i$. For integers $k\geq 0$ and $1\leq p\leq \infty$, a function $f: \cM \to \RR$ is in the Sobolev space $W^{k,p}(\cM)$ if
\begin{align*}
\|f\|_{W^{k,p}(\cM)}:=\sum_{i=1}^{C_{\cM}} \|(f\rho_i)\circ\phi_{i}^{-1}\|_{W^{k,p}(\phi_i(U_i))}<\infty.
\end{align*}
\end{definition}
Since $\cM$ is compact, a finite altas exists on $\cM$. Besides, we introduce the partition of unity $\rho_i$ to follow the standard definition in \citet[Definition 13.4]{tu2010introduction}. The existence of a smooth partition of unity is shown in Appendix \ref{appendix.manifold}. From Definition \ref{def.sobolevM}, we observe that a Sobolev function on $\cM$ is locally Sobolev on each chart.

\vfull

\subsection{Convolutional Residual Networks}\label{sec:convresnet}
\vfull

We consider one-sided stride-one convolution in our network. Let $\cW=\{\cW_{j,k,l}\}\in \RR^{C'\times K\times C}$ be a filter where $C'$ is the output channel size, $K$ is the filter size and $C$ is the input channel size. For $Z\in \RR^{D\times C}$, the convolution of $\cW$ with $Z$ gives $Y = \cW * Z \in \RR^{D\times C'}$ with
\begin{align*}
Y_{i,j}=\sum_{k=1}^K \sum_{l=1}^{C} \cW_{j,k,l} Z_{i+k-1,l},
\end{align*}
where we set $Z_{i+k-1,l}=0$ for $i+k-1>D$. See a graphical demonstration in Figure \ref{fig.conv}(a).

In this paper, we study convolutional residual networks (ConvResNets) equipped with the rectified linear unit ($\ReLU$) activation function ($\ReLU(z)=\max(z,0)$).
The ConvResNet we consider consists consecutively of a padding layer, several residual blocks, and finally a fully connected output layer.

Given an input vector $\xb \in\RR^{D}$, the network first applies a padding operator $P:\RR^{D}\rightarrow \RR^{D\times C}$ for some integer $C\geq 1$ such that
$$
Z=P(\xb)=\begin{bmatrix}
	\xb & \zero &\cdots & \zero 
\end{bmatrix} \in \RR^{D \times C}.
$$
Then the matrix $Z$ is passed through $M$ residual blocks. To ease the notation, we denote the input matrix to the $m$-th block as $Z_m$ and its output as $Z_{m+1}$ (Consequently, $Z_1 = Z$).

In the $m$-th block, let
$\cW_m=\{ \cW_m^{(1)},...,\cW_m^{(L_m)}\}$ and $\cB_m=\{B_m^{(1)},...,B_m^{(L_m)}\}$ be a collection of filters and biases of proper sizes.
The $m$-th residual block maps its input matrix $Z_m$ from $\RR^{D\times C}$ to $\RR^{D\times C}$ by the operator
$$
\Conv_{\cW_m,\cB_m}+\id,
$$
where $\id$ is the identity mapping (also known as the shortcut connection) and
\begin{align} 
&\Conv_{\cW_m,\cB_m}(Z_m)=\ReLU\Big( \cW_m^{(L_m)}*\cdots *\ReLU\left(\cW_m^{(1)}*Z_m+B_m^{(1)}\right)\cdots +B_m^{(L_m)}\Big),
 \label{eq.conv}
\end{align}
with $\ReLU$ applied entrywise. We denote the mapping from input $\xb$ to the output of the $M$-th residual block as
\begin{align}
Q(\xb)=&\left(\Conv_{\cW_M,\cB_M}+\id\right)\circ\cdots\circ \left(\Conv_{\cW_1,\cB_1}+\id\right)\circ P(\xb).
\label{eq.cnnBlock}
\end{align}

\begin{figure}[th!]
	\centering
	\subfloat[Convolution.]{\includegraphics[height=3.2cm]{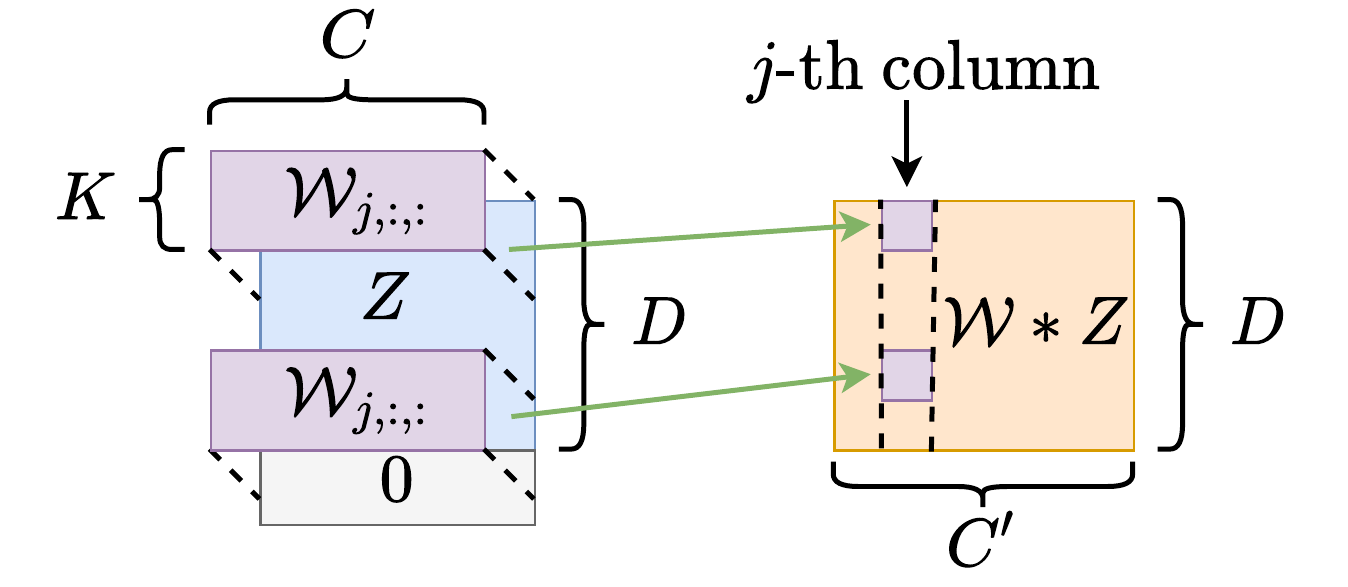}}
	\subfloat[A residual block.]{\includegraphics[trim={3cm 0 3cm 0},clip,height=3.2cm]{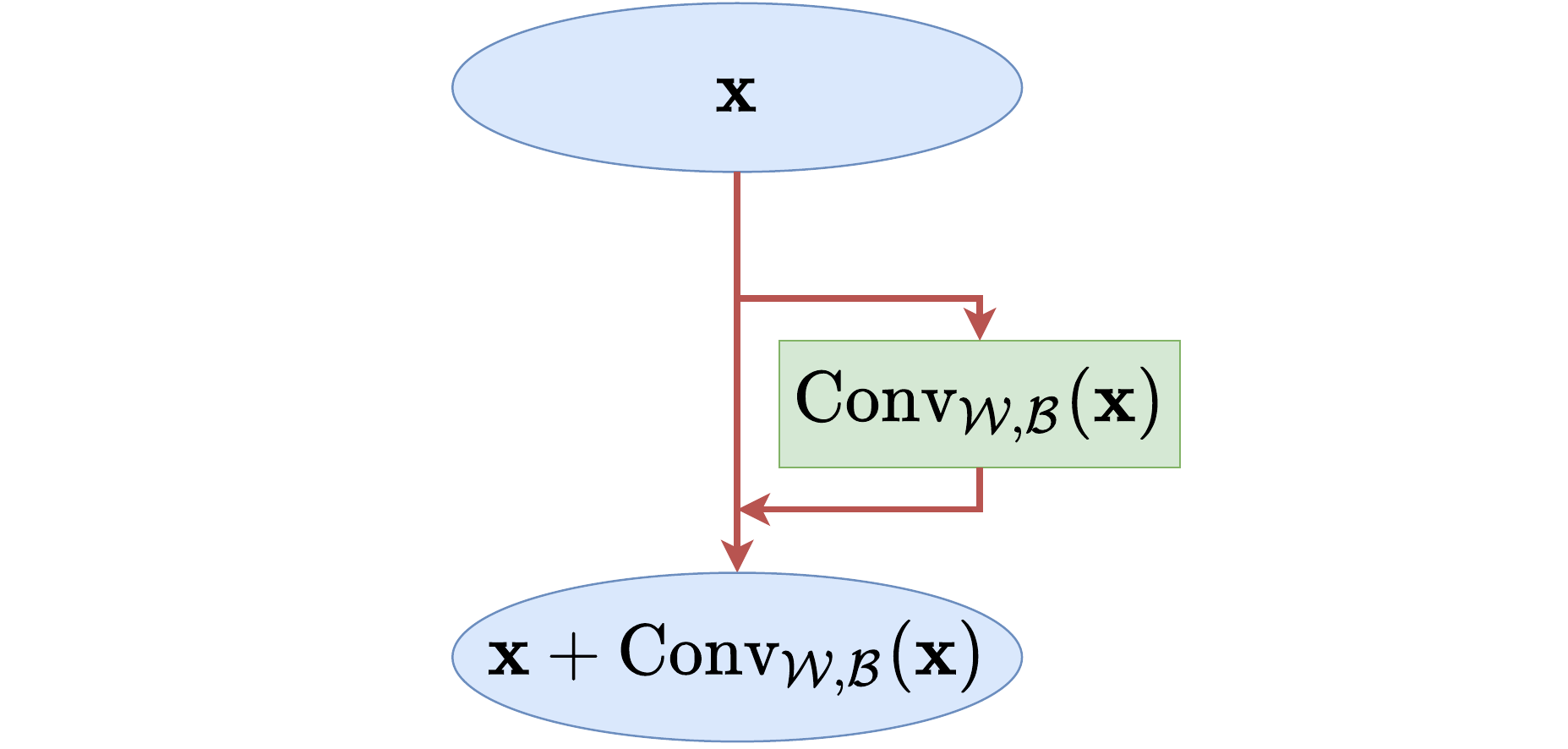}}
	\caption{(a) Convolution of $\cW *Z$, where the input is $Z\in  \RR^{D\times C}$, and the output is $\cW* Z \in \RR^{D\times C'}$. Here $\cW=\{\cW_{j,k,l}\}\in \RR^{C'\times K\times C}$ is a filter where $C'$ is the output channel size, $K$ is the filter size and $C$ is the input channel size.
		$\cW_{j,:,:}$ is a $D\times C$ matrix for the $j$-th output channel. (b) A convolutional residual block.
	}
	\label{fig.conv}
\end{figure}

Given \eqref{eq.cnnBlock}, a ConvResNet applies an additional fully connected layer to $Q$ and outputs
\begin{align*}
f(\xb)= W \otimes Q(\xb)+b,
\label{eq.cnn}
\end{align*}
where $W \in \RR^{D \times C}$ and $b \in \RR$ are a weight matrix and a bias, respectively, and $\otimes$ denotes sum of entrywise product, i.e., $W \otimes Q(\xb) = \sum_{i, j} W_{i, j} [Q(\xb)]_{i, j}$. To this end, we define a class of ConvResNets of the same architecture as
\begin{align}
&\cC(M,L,J,K,\kappa_1,\kappa_2)=\nonumber\\
&\big\{ f~|~ f(\xb)=W \otimes Q(\xb)+ b \mbox{ with } \|W\|_{\infty} \vee |b| \leq \kappa_2,  Q(\xb) \mbox{ in the form of \eqref{eq.cnnBlock} with $M$ residual blocks.} \nonumber\\
&\hspace{0.7cm} \mbox{The number of filters per block is bounded by }L;   \mbox{ filter size is bounded by } K; \nonumber \\
& \hspace{0.7cm} \mbox{the number of channels is bounded by }J;  \max_{m,l}\|\cW_m^{(l)}\|_{\infty} \vee \|B_m^{(l)}\|_{\infty} \leq \kappa_1  \big\}.
\vhalf
\end{align}

Here $\norm{\cdot}_\infty$ denotes the entrywise maximum norm, i.e., when the input argument is a vector, it returns the vector $\ell^\infty$ norm; when the input is a matrix or a tensor, it returns the maximum magnitude of its entries, e.g., for a $3$-dimensional tensor $\cW$, $\norm{\cW}_{\infty} = \max_{j,k,l} |\cW_{j,k,l}|$. 


\vfull
\section{Approximation in Euclidean Space}\label{sec:euclidean}
\vhalf
Consider a Sobolev function class defined on a unit hypercube $(0, 1)^D$. We aim to use convolutional residual networks for approximating functions in the target class in terms of the $W^{s, p}$ norm. Here $p$ is a positive integer and $s$ can vary in $[0, 1]$; in particular, $s = 0$ corresponds to function value approximation, and $s = 1$ resembles the result Section \ref{sec:intro}. We formally define our target function class as a Sobolev norm ball.
\vhalf
\begin{assumption}\label{assum.f}
Let $\alpha\geq 2, 1\leq p\leq +\infty$ be integers. Assume the target function $f$ satisfies
\begin{align*}
f\in W^{\alpha,p}\left((0,1)^D\right) \quad \text{and} \quad \|f\|_{W^{\alpha,p}\left((0,1)^D\right)}\leq 1.
\end{align*}
\end{assumption}
We set the norm ball of radius $1$ for the sake of simplicity, while the results in the sequel hold for any constant radius. We also let $\alpha \geq 2$ for techincal convenience. In the following theorem, we show that ConvResNets can approximate any functions in a Sobolev norm ball in terms of $W^{s, p}$ norm ($s \leq 1$). The approximation error is obtained as a function of the network configuration.
\begin{theorem}\label{thm.D}
For any positive integers $K \in [2, D]$, $\widetilde{M}$, and $\widetilde{J}>0$, we choose 
\begin{align*}
&L=O(\log (\widetilde{M}\widetilde{J})), ~J=O(\widetilde{J}), ~\kappa_1=O((\widetilde{M}\widetilde{J})^{1/D}),  \ \kappa_2=O((\widetilde{M}\widetilde{J})^{1/D}), ~M=O(\widetilde{M}).
\end{align*}
Then given $s \in [0, 1]$, the ConvResNet architecture $\cC(M,L,J,K,\kappa_1,\kappa_2)$ can approximate any function $f$ satisfying Assumption \ref{assum.f}, i.e., there exists $\widetilde{f} \in \cC(M,L,J,K,\kappa_1,\kappa_2)$ with
\begin{align*}
	\|\widetilde{f}-f\|_{W^{s,p}\left((0,1)^D\right)} \leq C_1(\widetilde{M}\widetilde{J})^{-\frac{\alpha-s}{D}}
\end{align*}
for some constant $C_1$ depending on $D,\alpha,p$.
\end{theorem}
Theorem \ref{thm.D} says that the approximation power of ConvResNet amplifies as its width and depth increase. To better interpret the result, we choose $s = 1$ and $p = \infty$, which corresponds to simultaneously approximating function value and first-order derivatives.
\vhalf 
\begin{corollary}\label{coro.D}
In the setup of Theorem \ref{thm.D}, taking $s = 1$ and $p = \infty$, the ConvResNet architecture $\cC(M,L,J,K,\kappa_1,\kappa_2)$ can approximate any $f$ satisfying Assumption \ref{assum.f} up to first-order, i.e., there exists $\tilde{f} \in \cC(M,L,J,K,\kappa_1,\kappa_2)$ with
\begin{align}
\norm{\tilde{f} - f}_\infty  \leq C_2(\widetilde{M}\widetilde{J})^{-\frac{\alpha-1}{D}} \quad \text{and} \quad \sup_i~ \left\| \frac{\partial \tilde{f}}{\partial x_i} - \frac{\partial f}{\partial x_i}\right\|_\infty  \leq C_2(\widetilde{M}\widetilde{J})^{-\frac{\alpha-1}{D}}, \nonumber
\end{align}
where the constant $C_2$ depends on $D$ and $\alpha$. In particular, we have Lipschitz continuity bound
\begin{align*}
\norm{\tilde{f}}_{\rm Lip} \leq 1 + C_2 \sqrt{D} (\tilde{M}\tilde{J})^{-\frac{\alpha-1}{D}}.
\end{align*}
\end{corollary}

Theorem \ref{thm.D} and Corollary \ref{coro.D} have rich implications.

\noindent {\bf Large network for smooth approximation}. Taking $s = 0$ in Theorem \ref{thm.D} recovers function approximation in terms of $L_\infty$ norm. The corresponding approximation error scales as $O((\tilde{M} \tilde{J})^{-\frac{\alpha}{D}})$. A quick comparison to Corollary \ref{coro.D} indicates that in order to additionally capture the first-order information of a target function, large network is needed to achieve the same function value error bound.

\noindent {\bf Arbitrary width and depth}. \citet{guhring2020error,hon2021simultaneous} provide approximation guarantees of feedforward networks in terms of $W^{s, p}$ norm. Despite different network architectures, we remark that our theory covers general networks with arbitrary width and depth. More specifically, for a given approximation error $\epsilon$, \citet{guhring2020error} set the network depth and width as $O(\log 1/\epsilon)$ and $O(\epsilon^{-D/(\alpha - s)})$, respectively. Yet in our result, we only need to ensure $\tilde{M}\tilde{J} = O(\epsilon^{-D/(\alpha - s)})$, which does not require any scaling relation between $\tilde{M}$ and $\tilde{J}$. 

Theorem \ref{thm.D} can be used as a tool to analyze the empirical residual error. Specifically, assume the response in the data set contains bounded zero--mean noise, we have the following probability bound on the upper bound of the empirical residual error (see a proof in Appendix \ref{sec.proof.prob})
\begin{theorem}\label{thm.prob}
	Let $\{(\xb_i,y_i)\}_{i=1}^n$ be a given data set where $\xb_i$'s are i.i.d. samples from some distribution defined on $[0,1]^D$ and 
	$$
	y_i=f(\xb_i)+\xi_i
	$$
	with i.i.d. noise $\xi_i$'s satisfying $\EE[\xi_i]=0$ and $|\xi_i|\leq \sigma$ for all $i=1,...,n$. Assume $f$ satisfy Assumption \ref{assum.f} with $p=+\infty$. For $0<\varepsilon<\min\{\sigma,1\}$, let $\cC=\cC(M,L,J,K,\kappa_1,\kappa_2)$ be the network architecture in Theorem \ref{thm.D} with $\widetilde{M}\widetilde{J}=\left(\frac{\varepsilon}{C_1}\right)^{-D/\alpha}=O(\varepsilon^{-D/\alpha})$. We have 
	\begin{align}
		&\PP\bigg(\exists  \widetilde{f}\in \cC : \left\|\widetilde{f}\right\|_{\rm Lip}\leq 1+\sqrt{D}\varepsilon^{\frac{\alpha-1}{\alpha}}\ \mbox{ and }  \frac{1}{n} \sum_{i=1}^n (\widetilde{f}(\xb_i)-y_i)^2\leq 2\varepsilon^2+\sigma^2\bigg)  \geq 1-\exp\left( -\frac{3n\varepsilon^2}{104\sigma^4} \right).
	\end{align} 
\end{theorem}
Theorem \ref{thm.prob} implies that with high probability, larger network architectures ensure the existence of a network that has small empirical residual error as well as certain smoothness, i.e., a bounded Lipschitz constant whcih is close to that of the underlying function. Our result is an upper bound counterpart of \citet[Theorem 3]{bubeck2021universal}, in which a high probability lower bound of the Lipschitz constant is derived.

\noindent {\bf Connection to adversarial robustness}. Consider, for example, the supervised learning scenario. Noisy or noiseless response is generated by a ground truth function satisfying Assumption \ref{assum.f}. Corollary \ref{coro.D} then indicates the existence of a properly large ConvResNet capable of smoothly approximating the data model, and the network's Lipschitz constant is approximately that of the ground truth function. Such Lipschitz continuity should be considered nearly optimal, in viewing of the smoothness of the ground truth function. The network's Lipschitz continuity closely relates to adversarial risk \cite{uesato2018adversarial,zhao2021adversarially} defined as
\begin{definition}[Adversarial risk]
	Given a data distribution $\rho$, and a loss function $l(\cdot,\cdot)$, for a positive constant $\delta>0$, we define the adversarial risk of a network $\widetilde{f}$ as
	\begin{align}
		R(\widetilde{f},\delta)=\EE_{(\xb,y)\in \supp(\rho)} \left[\sup_{\xb'\in B_{\delta}(\xb)} \ell\left(\widetilde{f}(\xb'),y\right)\right],
	\end{align}	
where $B_{\delta}(\xb)$ is the Euclidean ball with radius $\delta$ centered at $\xb$.
\end{definition}
In the case $\delta=0$, the adversarial risk $R(\widetilde{f},0)$ reduces to the population risk $\EE_{(\xb,y)\in \supp(\rho)} \left[ \ell\left(\widetilde{f}(\xb),y\right)\right]$. Based on Theorem \ref{thm.D} and Corollary \ref{coro.D}, we have the following theorem on adversarial risk (see a proof in Appendix \ref{sec.proof.adversarial}):
\begin{theorem}\label{thm.adversarial}
	Let $\rho$ be a data distribution defined on $[0,1]^D\times [-R,R]$ for some constant $R$ and $l(\cdot,\cdot)$ be a loss function with Lipschitz constant $L_{\rm Lip}$. Denote the population risk minimizer by $f$:
	\begin{align}
		f=\argmin_g \EE_{(\xb,y)\in \supp(\rho)} l(g(\xb),y).
	\end{align}
	Assume $f$ satisfies Assumption \ref{assum.f} with $p=+\infty$.
	For $0<\varepsilon<1$, let $\cC(M,L,J,K,\kappa_1,\kappa_2)$ be the network architecture in Theorem \ref{thm.D} with $\widetilde{M}\widetilde{J}=\left(\frac{\varepsilon}{C_1}\right)^{-D/\alpha}=O(\varepsilon^{-D/\alpha})$. Then there exists $\widetilde{f}\in \cC(M,L,J,K,\kappa_1,\kappa_2)$ so that
	\begin{align}
		\|\widetilde{f}-f\|_{\infty}\leq \varepsilon,\quad \left\|\widetilde{f}\right\|_{\rm Lip}\leq 1+\sqrt{D}\varepsilon^{\frac{\alpha-1}{\alpha}}
	\end{align}
	and
	\begin{align}
		R(\widetilde{f},\delta)\leq R(\widetilde{f},0)+L_{\rm Lip}\left(1+\sqrt{D}\varepsilon^{\frac{\alpha-1}{\alpha}}\right)\delta.
	\end{align}
\end{theorem}
In Theorem \ref{thm.adversarial}, the difference between the adversarial risk and population risk depends on the Lipschitz constant of the network $\widetilde{f}$, the Lipscthiz constant of the loss function and the adversarial parameter $\delta$. It implies that large networks can give rise to smooth functions with a small adversarial risk, i.e., adversarially robust.
This partially explains the empirical observation that large networks are often smooth with respect to input, and hence, tend to have better robustness. However, how to use practical training algorithms to find such networks remains curiously unclear.

\vfull
\section{Approximation on Manifold}
Theorem \ref{thm.D} indicates a curse of data dimensionality: When data dimension $D$ is large, such as image data, Theorem \ref{thm.D} converges extremely slowly and becomes less attractive.
Motivated by applications, we model data as a low-dimensional Riemannian manifold $\cM$ and extend our approximation theory to functions defined on $\cM$. We will show that ConvResNet is adaptable to manifold structures. We first impose some mild regularity conditions.
\begin{assumption}\label{assum.M}
$\cM$ is a $d$-dimensional compact Riemannian manifold isometrically embedded in $\RR^D$. It's range is bounded by $B$, i.e., there exists a constant $B>0$ such that for any $\xb\in\cM$, we have $\|\xb\|_{\infty}\leq B$.
\end{assumption}
Besides boundedness, we characterize the curvature of manifold by the following geometric notion.


\begin{definition}[Reach \citep{federer1959curvature,niyogi2008finding}]
Define the set
\begin{align*}
G=\left\{ \xb\in \RR^D: \exists\mbox{ distinct } \pb,\qb\in\cM \mbox{ such that }  d(\xb,\cM)=\|\xb-\pb\|_2=\|\xb-\qb\|_2\right\}.	
\end{align*}
Then the reach of $\cM$ is defined as
$$
{\rm reach}(\cM) =\inf_{\xb \in \cM} \ \inf_{\yb\in G}\|\xb-\yb\|_2.
$$
\end{definition}
To roughly put, a large reach implies that the manifold is flat. While a manifold with a small reach can be highly zigzagging. Therefore, the reach is highly relevant to the difficulty of capturing the local structures on a manifold. We assume a positive reach on $\cM$.
\begin{assumption}\label{assum.reach}
The reach of $\cM$ is $\tau>0$.
\end{assumption}
Similar to Section \ref{sec:euclidean}, we consider a Sobolev norm ball on $\cM$ as target function class.
\begin{assumption}\label{assum.M.f}
	Let $\alpha\geq 2$ be an integer. Assume the target function $f$ satisfies
\begin{align*}
f\in W^{\alpha,\infty}\left(\cM\right) \quad \text{and} \quad \|f\|_{W^{\alpha,\infty}\left(\cM\right)}\leq 1.
\end{align*}
\end{assumption}
We now present a counterpart of Theorem \ref{thm.D}, showing an efficient approximation of functions in a Sobolev norm ball on $\cM$.
\begin{theorem}\label{thm.M}
For any positive integers $K \in [2, D]$, $\widetilde{M}$, and $\widetilde{J}>0$, we choose 
\begin{align*}
&\hspace{0.35in} L=O(\log (\widetilde{M}\widetilde{J}))+D, ~J=O(D\widetilde{J}),   \  \kappa_1=O((\widetilde{M}\widetilde{J})^{1/d}), ~\kappa_2=O((\widetilde{M}\widetilde{J})^{1/d}), ~M=O(\widetilde{M}).
\end{align*}
Then given $k \in \{0, 1\}$, the ConvResNet architecture $\cC(M,L,J,K,\kappa_1,\kappa_2)$ can approximate any function $f$ satisfying Assumption \ref{assum.M.f}, i.e., there exists $\widetilde{f} \in \cC(M,L,J,K,\kappa_1,\kappa_2)$ with
\begin{align*}
	\|\widetilde{f}-f\|_{W^{k,\infty}\left(\cM\right)} \leq C_3 (\widetilde{M}\widetilde{J})^{-\frac{\alpha-k}{d}},
\end{align*}
where constant $C_3$ depends on $d,\alpha,B,\tau$, and the surface area of $\cM$.
\end{theorem}
As can be seen, the approximation error decays at a rate only depending on intrinsic data dimension $d$, which is a significant improvement over Theorem \ref{thm.D} given $d \ll D$. We also note that the size of ConvResNet has a weak dependence on $D$, yet it is inevitable due to the residual connection preserves input dimensionality.

Theorem \ref{thm.M} can be viewed as further results of recent advances on the adaptability of neural networks for approximating functions on low-dimensional structures. In particular, \citet{chen2019efficient} and \citet{schmidt2019deep} share a very similar setup as Theorem \ref{thm.M}, and established function value approximation theories.


\section{Numerical Experiments}
We verify our theory by numerical experiments. Due to the complex structure of convolutional residual networks, directly estimating the Lipschitz constant is rather difficult. We instead testing the adversarial robustness as an indication of the network smoothness.

 We consider the TRADES model which uses a data driven smoothness regularization and encourages model smoothness. By keeping the same clean testing accuracy, we can compare model smoothness through the robust testing accuracy. We follow the setup in TRADES \citep{zhang2019theoretically}, and report the performance of WideResNet \citep{zagoruyko2016wide} with different widening factor (WF) and number of convolutional layers per residual block (we term as ``depth'' in the sequel). We use the CIFAR-10 data set. Hyperparameters in training are set as follows: perturbation diameter $\epsilon = 0.031$ under the $\ell_\infty$ norm, step size for generating perturbation $0.007$, number of iterations $10$, learning rate $0.1$, batch size $b = 128$ and run $76$ epochs on the training dataset. We run the White-box attacks by applying PGD attack with $20$ iterations (PGD-20) and the step size is $0.003$. We report the robust accuracy $\mathcal{A}_{\rm rob}$  and the natural accuracy $\mathcal{A}_{\rm nat}$ on the test data set. 


The training objective is
\begin{align*}
	\min _{f} \EE_{(\xb,y)\sim\cD}\mathcal{L}(f(\xb), y)+\max _{\norm{\tilde{\xb}- \xb}_{\infty}\leq\epsilon} \mathcal{R}\left(f(\xb), f\left(\tilde{\xb}\right)\right) / \lambda,
\end{align*}
where $\cL$ is the cross entropy loss, $\cR$ is the KL-divergence, $\xb$ is the clean input, $\tilde{\xb}$ is the adversarial input, $y$ is the label, $\lambda$ is the tuning parameter controlling the strength of the regularizer, and $\cD$ denotes the training dataset $\{\xb_i,y_i\}_{i=1}^n$.

For a fair comparison, we tune $\lambda$ such that networks of different sizes achieve approximately the same natural accuracy. This can be understood as achieving approximately the same $L_\infty$ approximation error to the data model. As can be seen in Table \ref{tab:experiment}, $\cA_{\rm nat}$ of different models about matches the performance in \citet{zhang2019theoretically}, indicating the network has been sufficiently trained. By comparing the robust accuracy $\mathcal{A}_{\rm rob}$, we observe that wider and deeper WideResNet attains better robustness. When fixing the depth, a wider network can achieve a higher robust accuracy. Similarly, when fixing the widening factor, a deeper network can achieve a higher robust accuracy.
\begin{table}[h]
    \centering
    \begin{tabular}{c|c|p{7em}|p{7em} }
	    \hline
	    Depth & WF & $\mathcal{A}_{\rm nat}$ & $\mathcal{A}_{\rm rob}$ \\
	    \hline
	    \multirow{3}{2em}{16} & 1 & $78.87\pm0.47$\%  & $34.31 \pm 0.45$\% \\
	           & 2 &  $79.34 \pm 0.28$ \%  & $46.14 \pm 0.21$\% \\
	           & 4 &  $79.97 \pm 0.04$\%  & $51.40 \pm 0.16$\% \\
	    \hline
	    \multirow{3}{2em}{22} & 1 & $78.51 \pm 0.25$\%  & $41.47 \pm 0.11$\% \\
	           & 2 &  $79.49 \pm 0.48$\%  & $49.63 \pm 0.07$\% \\
	           & 4 &  $80.81 \pm 0.44$\%  &  $53.36 \pm 0.21$\% \\
	    \hline
	    \multirow{3}{2em}{28} & 1 & $79.46 \pm 0.06$\%  & $43.33 \pm 0.57$\% \\
	           & 2 &  $79.01 \pm 0.11$\%  & $50.85 \pm 0.07$\% \\
	           & 4 &  $80.90 \pm 0.71$\%  & $54.45 \pm 0.14$\% \\
	    \hline
	    \multirow{3}{2em}{34} & 1 & $78.58 \pm 0.09$\%  & $46.14 \pm 0.16$\% \\
	           & 2 &  $79.29 \pm 0.35$\%  & $51.63 \pm 0.28$\% \\
	           & 4 &  $80.79 \pm 0.71$\%  & $55.28 \pm 0.35$\% \\
	    \hline
	    \end{tabular}
    \caption{Performance of Wide Residual Networks with different widening factors and depths under PGD-20 attacks.}
    \label{tab:experiment}
\end{table}

\section{Proof Sketch}
We highlight key steps in establishing Theorem \ref{thm.D} and \ref{thm.M} in this section. Full proofs are deferred to Appendix \ref{sec.proof.D} and \ref{sec.proof.M}, respectively.
\subsection{Proof Sketch of Theorem \ref{thm.D}}
The main idea consists of two stages: 1) Approximating target function $f$ in terms of $W^{s, p}$ norm using a sum of averaged Taylor polynomials; 2) Implementing the sum of averaged Taylor polynomials by a given width and depth ConvResNet up to a certain error. In stage 1), we rely on tools from the finite element anaylsis to quantify approximation error. In stage 2), we first represent polynomials using convolutional networks, and then assemble them according to the specified width and depth as a ConvResNet. We dive into the following four steps.

\noindent\textbf{Step 1: Decompose $f$ using a partition of unity.} Given the network size parameter $\tilde{M}$ and $\tilde{J}$, we define a partition of unity $\{\phi_j\}_{j=1}^{N^D}$ on $(0,1)^D$ for an integer $N = O((\tilde{M}\tilde{J})^{1/D})$, so that each $\phi_j$ is supported on a small hypercube of edge length $\frac{4}{3N}$. The function $f$ is decomposed into $f=\sum_{j=1}^{N^D} f_j$ with $f_j=f\phi_j$. See Figure \ref{fig.Omega.PoU}(a) for an illustration.

\noindent\textbf{Step 2: Averaged Taylor polynomial approximation.} Each $f_j$ is a Sobolev function, which may not have classical derivatives but weak derivatives. Similar to approximating differentiable functions by Taylor polynomials, we approximate $f_j$ by an averaged Taylor polynomial $\widehat{f}_j$, which is defined in an integral form and indeed is a polynomial. The approximation error of averaged Taylor polynomial is similar to that of using Taylor polynomial, and can be found in Lemma \ref{lem.ATE.error}.

\noindent\textbf{Step 3: Network implementation.} As shown in Lemma \ref{lem.multiplication.CNN} and \ref{lem.product}, CNN can approximate multiplication and compositions of muliplications well. Since a polynomial is a sum of compositions of multiplication, each $\widehat{f}_i$ can be approximated by a sum of $O(1)$ CNNs, and therefore $\sum_{i=1}^{N^D} \widehat{f}_i$ is approximated by a sum of $O(N^D)$ CNNs, each of which has width of $O(1)$. We prove in Lemma \ref{lem.CNN.adap} that such a sum can be realized by a sum of $\widetilde{M}$ CNNs with width $\widetilde{J}$. The new sum can be realized by a ConvResNet with $\widetilde{M}$ residual blocks (Lemma \ref{lem.cnn.convresnet}), where each summand corresponds to a residual block and the sum is realized using skip-layer connections. 
	
\noindent\textbf{Step 4: Error estimation.} To estimate the approximation error of $\widetilde{f}$, we decompose the error as
\begin{align}
\|\widetilde{f}-f\|_{W^{s,p}(0,1)^D} \leq &\sum_{j=1}^{N^D}\|\widetilde{f}_j-\widehat{f}_j\|_{W^{s,p}((0,1)^D)}+ \sum_{j=1}^{N^D}\|\widehat{f}_j-f_j\|_{W^{s,p}((0,1)^D)}.
\label{eq.D.error}
\end{align}
On the right-hand side of (\ref{eq.D.error}), the second term is the approximation error of averaged Taylor polynomial, whose upper bound is given by Lemma \ref{lem.ATE.approx.error}. 

The first term is the network implementation error. We derive an upper bound of it in Lemma \ref{lem.net.error}. In the proof of Lemma \ref{lem.net.error}, we first derive an upper bound with respect to the $W^{k,p}$ norm for $k=0,1$. The case $k=0$ corresponds to the error of function value approximation, and the case $k=1$ corresponds to the error of first order weak derivative approximation. Note that each $\widehat{f}_j$ is a polynomial, and each $\widetilde{f}_j$ consists of compositions of $\widetilde{\times}$, the network approximation of multiplication $\times$. The error indeed is the approximation error of compositions of $\widetilde{\times}$. We first derive the $W^{k,\infty}$ approximation error of $\widetilde{\times}$ and then show that compositions of $\widetilde{\times}$ have $W^{k,p}$ approximation errors of the same order.
After the upper bounds of $W^{0,p}$ and $W^{1,p}$ errors are derived, these upper bounds are generalized to $W^{s,p}$ errors using an argument on interpolation spaces, which is discussed in Appendix \ref{sec.interpolation}. 

Combining the upper bounds of both terms in (\ref{eq.D.error}) gives rise to the total approximation error as a function of $N$. Utilizing the relation $\widetilde{M}\widetilde{J}=O(N^D)$, we can further express the approximation error in terms of number of blocks and width of the ConvResNet. 


\begin{figure}[th!]
	\centering
	\subfloat[Partition of unity on $(0,1)^D$.]{\includegraphics[height=1.3in]{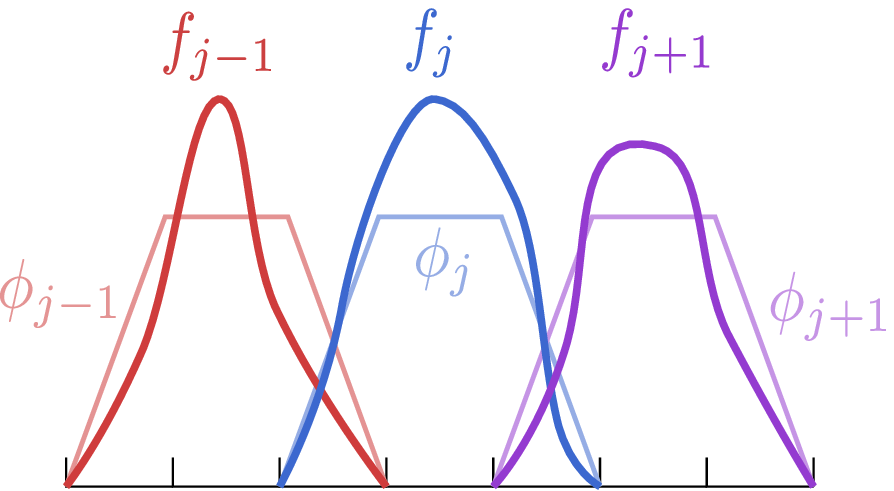}}\hspace{1cm}
	\subfloat[Partition of unity on $\cM$.]{\includegraphics[height=1.3in]{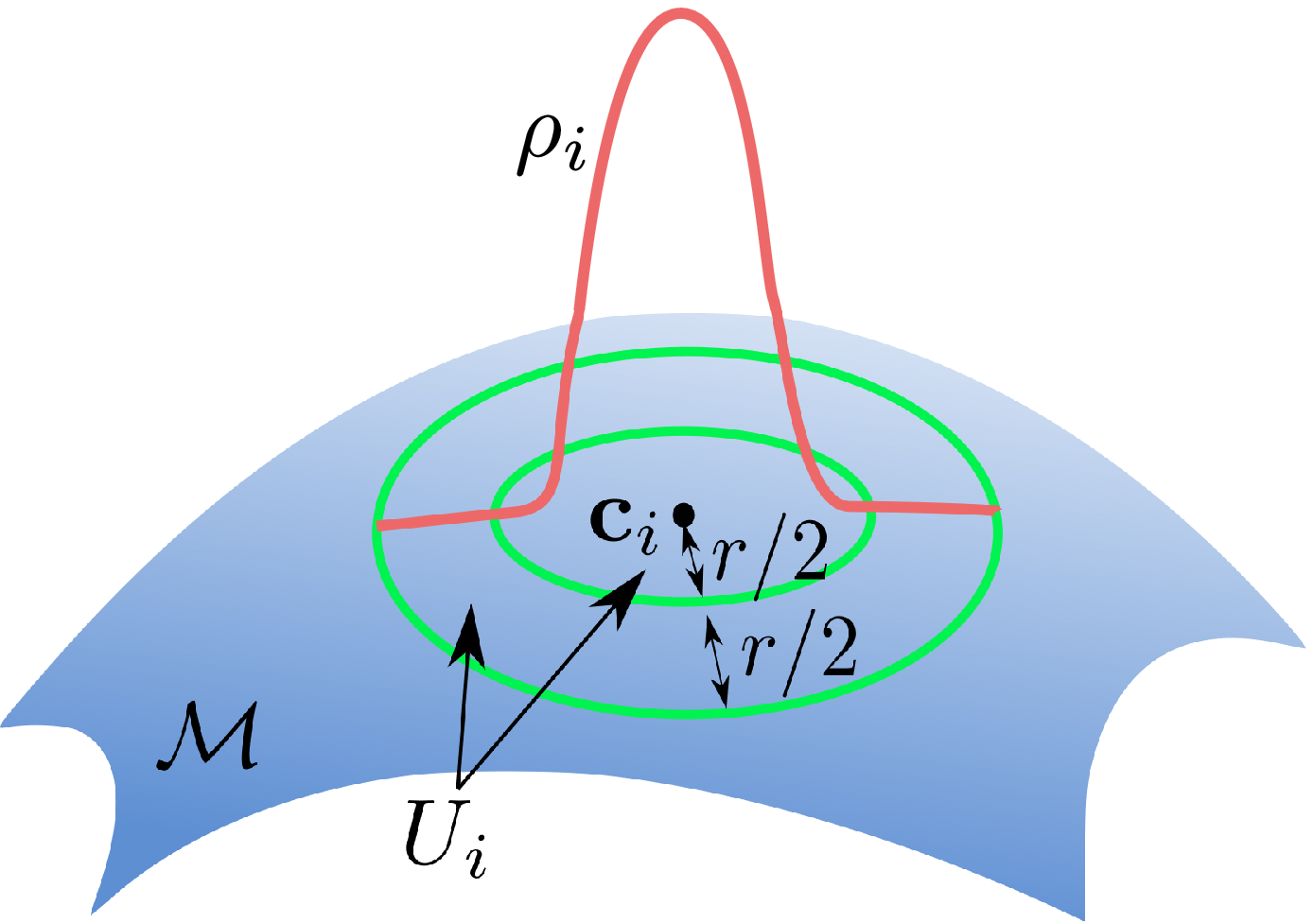}}
	\caption{(a) Illustration of $\phi_j$'s and $f_j$'s in Step 1 of the proof of Theorem \ref{thm.D}. (b) Illustration of the construction of charts and paritition of unity in Step 1 of the proof of Theorem \ref{thm.M}. The red curve represents a cross section of $\rho_i$.}
	\label{fig.Omega.PoU}
\end{figure}


\subsection{Proof Sketch of Theorem \ref{thm.M}}
We exploit the geometric nature of manifold $\cM$ and Sobolev functions on it to prove Theorem \ref{thm.M}. By an explicit construction of a finite atlas on $\cM$ based on the curvature condition in Assumption \ref{assum.reach}, we first restrict ourselves to a single chart on $\cM$. Recall Definition \ref{def.sobolevM} that a Sobolev function $f$ on $\cM$ is locally Sobolev on a chart. We are thus, able to locally approximate $f$ on each chart by the results in Theorem \ref{thm.D}. However, the main challenge stems from combining these local approximations to obtain a global guarantee. This requires to determine which charts a given input belongs to. We develop a chart determination sub-network for approximating indicator functions of charts, nonetheless, its Lipschitz continuity is troublesome due to the sharp jump on the boundary of a chart. We resolve such an issue by carefully constructing a partition of unity vanishing at a neighborhood of the boundary of charts. We provide more details in the following four steps.

\textbf{Step 1: Decompose $f$ using an atlas and partition of unity of $\cM$.} We first construct an atlas and a partition of unity of $\cM$ so that each function in the partition of unity is compactly supported in a chart (Lemma \ref{lem.parunity}). To construct an atlas of $\cM$, we use a set of $D$-dimensional Euclidean balls  $\{B_{r/2}(\cbb_i)\}_{i=1}^{C_{\cM}}$ with centers $\{\cbb_i\}_{i=1}^{C_{\cM}}\subset \cM$ and radius $r/2$ satisfying $0<r<\tau/4$ to cover $\cM$. Since $\cM$ is compact, $C_{\cM}$ is finite. The collection of intersections between each ball and $\cM$, denoted by $\{\widetilde{U}_i\}_{i=1}^{C_{\cM}}$ with $\widetilde{U}_i=B_{r/2}(\cbb_i)\cap \cM$, forms an open cover of $\cM$. It is guaranteed that there exists a $C^{\infty}$ partition of unity $\{\rho_i\}_{i=1}^{\cM}$ so that $\rho_i$ is supported in $\widetilde{U}_i$ (Lemma \ref{lem.parunity.exist}). We then double the radius and denote $U_i=B_{r}(\cbb_i)\cap \cM$. The collection $\{U_i\}_{i=1}^{C_{\cM}}$ is also an open cover of $\cM$. Since $\widetilde{U}_i\subset U_i$, $\rho_i$ is compactly supported in $U_i$ and the distance between the support of $\rho_i$ and $\partial U_i$ is at least $r/2$. For each $U_i$, an orthogonal projection $\varphi_i$ with proper scaling and shifting, which projects any $\xb\in U_i$ to a tangent plane, is constructed so that $\varphi_i(U_i)\subset (0,1)^d$. 
See the proof of Lemma \ref{lem.parunity} for details. With this construction, we illustrate $U_i$ and $\rho_i$ in Figure \ref{fig.Omega.PoU}(b). We then focus on the atlas $\{U_i,\varphi_i\}_{i=1}^{C_{\cM}}$ and partition of unity $\{\rho_i\}_{i=1}^{\cM}$. We decompose $f$ as $f=\sum_{i=1}^{C_{\cM}} (f_i\circ\varphi_i^{-1})\circ \varphi_i$ with $f_i=f\rho_i$. 

\textbf{Step 2: Averaged Taylor polynomial approximation.} In the decomposition in Step 1, each $f_i\circ\varphi_i^{-1}$ is a Sobolev function compactly supported in $\varphi_i(U_i)\subset (0,1)^d$. Extend $f_i\circ\varphi_i^{-1}$ to $(0,1)^d$ by 0. The extended function has the same smoothness as $f_i\circ\varphi_i^{-1}$, and can be approximated by a sum of local averaged Taylor polynomials $\sum_{i=1}^{N^d} \widehat{f}_{i,j}$, as what has been done in the proof of Theorem \ref{thm.D}. 

\textbf{Step 3: Network implementation.} Each polynomial $\widehat{f}_{i,j}$ can be approximated by a CNN $\widetilde{f}_{i,j}$. Since we are only interested in the value of $\widetilde{f}_{i,j}\circ\varphi_i(\xb)$ when $\xb\in U_i$, we need to determine the chart it belongs to. We accomplish this by introducing a chart determination function $\mone_i(\xb)=\mone_{[0,r^2]}\circ d^2_i(\xb)$, where $\mone_{[0,r^2]}(a)$ is a step function which outputs 1 when $a\in[0,r^2]$ and outputs 0 otherwise, $d^2_i(\xb)$ computes the squared Euclidean distance between $\xb$ and $\cbb_i$. The squared distance function $d^2_i$ can be approximated by a CNN with high accuracy. To approximate the step function $\mone_{[0,r^2]}$, we construct a CNN which outputs 1 on $[0,r^2-\Delta]$, 0 on $[r^2,\infty)$ and is linear on $[r^2-\Delta,r^2]$ for some small $\Delta$. The CNN approximation of $\mone_i$ , denoted by $\widetilde{\mone}_i$, is illustrated in Figure \ref{fig.Omega.simple}(a). Our network approximation of $f$ is constructed as
\begin{align*}
\widetilde{f}(\xb)=\sum_{i=1}^{C_{\cM}} \sum_{j=1}^{N^d} \widetilde{\times}(\widetilde{f}_{i,j}\circ\varphi_i(\xb),\widetilde{\mone}_i(\xb)),
\end{align*}
where $\widetilde{\times}$ denotes the CNN approximation of multiplication. By Lemma \ref{lem.CNN.adap} and \ref{lem.cnn.convresnet}, $\widetilde{f}$ can be realized by a ConvResNet with $\widetilde{M}$ blocks and width of $O(\widetilde{J})$ as long as $\widetilde{M}\widetilde{J}=O(N^d)$.
	

\textbf{Step 4: Error estimation.} We decompose the error into two parts: 1) the error between $f$ and its averaged Taylor polynomial approximation, and 2) the error between the averaged Taylor polynomial and its network approximation, see (\ref{eq.M.err}) in Appendix \ref{sec.proof.M}. The first part can be bounded using Lemma \ref{lem.ATE.approx.error}.
The second part is characterized by the approximation error of $\widetilde{\times}$ for multiplication, of $\widetilde{f_{i,j}}$ for averaged Taylor polynomials, and of $\widetilde{\mone}_i$ for chart determination $\mone_i$. The first two errors can be bounded using techniques similar to those in the proof of Theorem \ref{thm.D}. 

\begin{figure}[th!]
	\centering
	\subfloat[Chart determination.]{\includegraphics[height=1.2in]{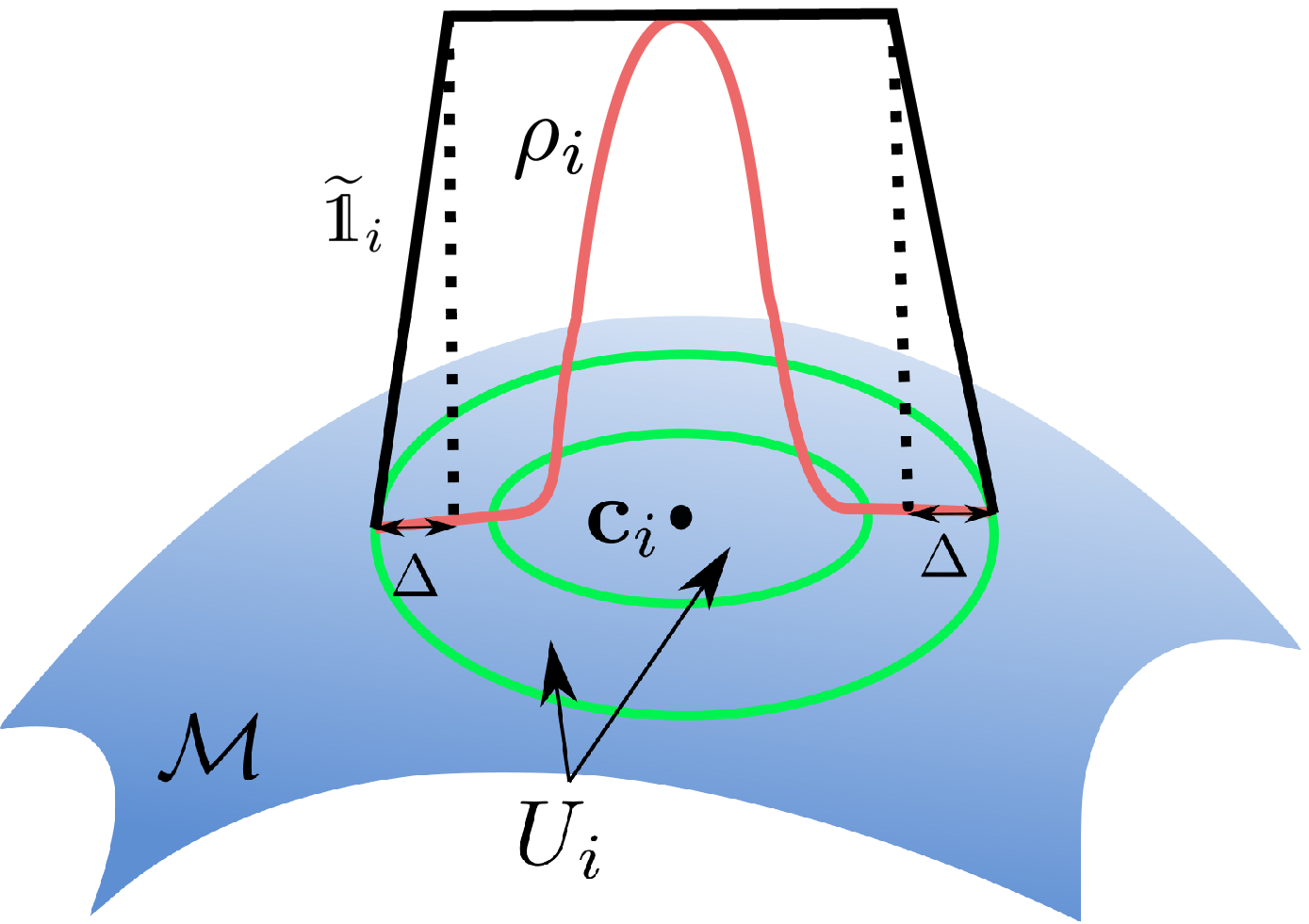}}\hspace{1cm}
	\subfloat[Projected region in $(0,1)^d$.]{\includegraphics[height=1.2in]{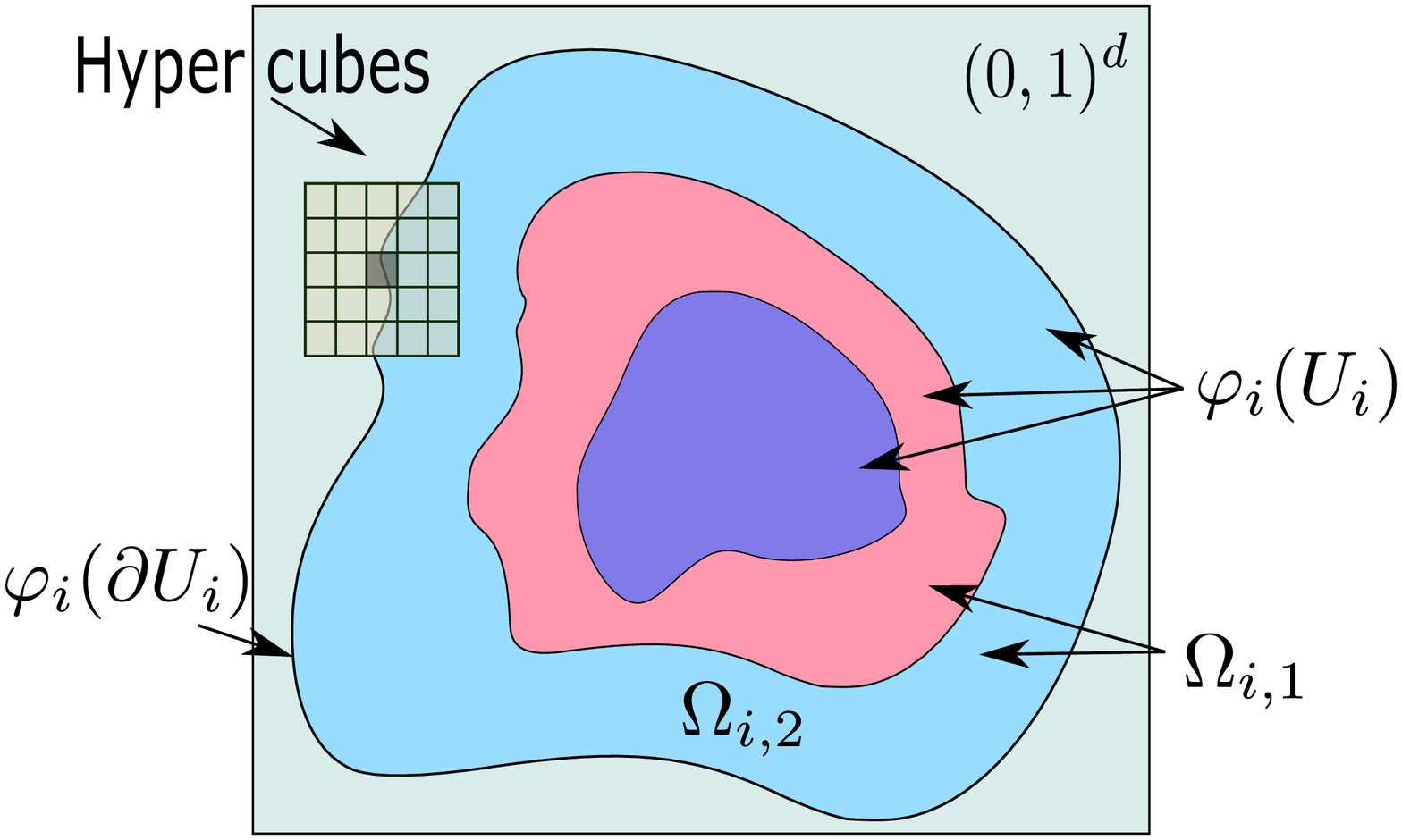}}
	\caption{(a) Illustration of an element of a chart and partition of unity. The red curve represents a cross section of $\rho_i$. (b) Illustration of the chart determination network $\widetilde{\mone}_i$. The black curve represents a cross section of $\widetilde{\mone}_i$. (c) Illustration of the projected regions in $(0,1)^d$.}
	\label{fig.Omega.simple}
\end{figure}

For the approximation error of $\widetilde{\mone}_i$, bounding its $W^{1,\infty}$ norm is the most challenging task. To derive an upper bound, one needs to bound
$|(\widetilde{f}_{i,j}\circ\varphi_i)\times (\partial (\widetilde{\mone}_i\circ \varphi_i^{-1})/\partial z_l)|$ for $l=1,...,d$ and $\zb\in \varphi_i(U_i)$. In our network construction, $\widetilde{\mone}_i\circ \varphi_i^{-1}$ is linear on a narrow band, denoted by $\Omega_{i,2}$, with width of $O(\Delta)$. Its weak derivative on the narrow band is of $O(1/\Delta)$, which blows up as $\Delta\rightarrow0$ and causes problems.
To eliminate the effect of $\Delta$, we show that the value of $\widetilde{f}_{i,j}\circ\varphi_i$ is small enough so that its product with $\partial (\widetilde{\mone}_i\circ \varphi_i^{-1})/\partial z_l$ does not blow up as $\Delta\rightarrow 0$.
Specifically, thanks to the fact that $f_i$ is compactly supported on $U_i$, we have $f_i\circ\varphi^{-1}$ is compactly supported on $\varphi_i(U_i)$. Therefore there exists another band $\Omega_{i,1}$ adjacent to $\varphi_i(\partial U_i)$ so that $f_i\circ\varphi^{-1}=0$ on $\Omega_{i,1}$. We choose $\Delta$ small enough so that $\Omega_{i,2}\subset \Omega_{i,1}$, and $\widehat{f}_{i,j}$ and all of its first order weak derivatives vanish on $\Omega_{i,2}$, see Figure \ref{fig.Omega.simple}(a) and (b) for illustrations. Note that $\widetilde{f}_{i,j}$ is an approximation of $\widehat{f}_{i,j}$. We can show that $\widetilde{f}_i=0$ on $\varphi_i(\partial U_i)$, and all of its first order weak derivatives on $\Omega_{i,2}$ are in the same order of other error terms. Since the width of $\Omega_{i,2}$ is of $O(\Delta)$, by Taylor's theorem, $|\widetilde{f}_{i,j}\circ\varphi_i|$ is bounded by a linear function of $\Delta$ on $\Omega_{i,2}$. With such a construction and proper choice of $\Delta$, the resulting upper bound is in the same order of those of other terms. See Lemma \ref{lem.A2} for details. 

Combining all of the error bounds, we can express the error in terms of $N$. Substituting the relation $\widetilde{M}\widetilde{J}=O(N^d)$ proves Theorem \ref{thm.M}.


\section{Conclusion}
We provide universal approximation theories of Convolutional Residual Networks in terms of Sobolev norms. Our theory applies to Sobolev function spaces defined on a high-dimensional hypercube or low-dimensional Riemannian manifold. We demonstrate that deep and wide ConvResNets can provide approximation with good first-order smoothness properties. This partially justifies why using large networks in practice often leads to better performance and robustness.

\section*{Acknowledgment}
The work of Hao Liu is partially supported by HKBU 162784 and HKBU 179356. The work of Wenjing Liao is partially supported by DMS 2012652 and NSF CAREER 2145167. The work of Wenjing Liao and Tuo Zhao is partially supported by DMS 2012652.

\bibliography{ref.bib}
\bibliographystyle{ims}

\appendix

\newpage
\onecolumn

\section*{Appendix}

\section{A Brief Introduction to Manifold}\label{appendix.manifold}
We introduce some concepts and quantities that characterize a low-dimensional Riemannian manifold. (Some are restatements of the main text for completeness.) These concepts and quantities are used in our theorems and proofs. We refer readers to \citet{lee2006riemannian,tu2010introduction} for more details.

Let $\cM$ be a $d$-dimensional manifold embedded in $\RR^D$ with $d\leq D$. The first concept related to manifolds is chart, which defines a local coordinate neighborhood of a manifold.
\begin{definition}[Chart]
	A chart on $\cM$ is a pair $(U, \phi)$ where $U \subset \cM$ is open and $\phi : U \to \RR^d,$ is a homeomorphism (i.e., bijective, $\phi$ and $\phi^{-1}$ are both continuous).
\end{definition}
In a chart $(U,\phi)$, $U$ is called a coordinate neighborhood and $\phi$ is a coordinate system on $U$.
A collection of charts which covers $\cM$ is called an atlas of $\cM$.
\begin{definition}[$C^k$ Atlas]
	A $C^k$ atlas for $\cM$ is a collection  of charts $\{(U_\alpha, \phi_\alpha)\}_{\alpha \in \cA}$ which satisfies $\bigcup_{\alpha \in \cA} U_\alpha = \cM$, and
	are pairwise $C^k$ compatible, i.e.,
	\begin{align*}
		&\phi_\alpha \circ \phi_\beta^{-1} : \phi_\beta(U_\alpha \cap U_\beta) \to \phi_\alpha(U_\alpha \cap U_\beta) \quad \textrm{and} \\
		& \phi_\beta \circ \phi_\alpha^{-1} : \phi_\alpha(U_\alpha \cap U_\beta) \to \phi_\beta(U_\alpha \cap U_\beta)
	\end{align*} are both $C^k$ for any $\alpha,\beta\in \cA$. An atlas is called finite if it contains finitely many charts.
\end{definition}

With the concept of atlas, we then define smooth manifolds:
\begin{definition}[Smooth Manifold] A smooth manifold is a manifold $\cM$ together with a $C^\infty$ atlas.
\end{definition}
Simple examples of smooth manifold include the Euclidean space, the torus and the unit sphere. $C^s$ functions on a smooth manifold $\cM$ are defined as follows:
\begin{definition}[$C^s$ functions on $\cM$]
	Let $\cM$ be a smooth manifold and $f:\cM \rightarrow \RR$ be a function on $\cM$. We say $f$ is a $C^s$ function defined on $\cM$, if for every chart $(U,\phi)$ on $\cM$, the function $f\circ \phi^{-1}: \phi(U)\rightarrow \RR$ is a $C^s$ function. \end{definition}

We next define the $C^\infty$ partition of unity which is an important tool for the study of functions on manifolds.
\begin{definition}[Partition of Unity]
	A $C^\infty$ partition of unity on a manifold $\cM$ is a collection of $C^\infty$ functions $\{\rho_{\alpha}\}_{\alpha\in\cA}$ with $\rho_\alpha: \cM \to [0,1]$ such that for any $\xb\in\cM$,
	\begin{enumerate}
		\item there is a neighbourhood of $\xb$ where only a finite number of the functions in $\{\rho_{\alpha}\}_{\alpha\in\cA}$ are nonzero, and
		\item $\displaystyle\sum_{\alpha\in\cA} \rho_\alpha(\xb) = 1$.
	\end{enumerate}
\end{definition}
An open cover of $\cM$ is called locally finite if every $\xb\in\cM$ has a neighbourhood that intersects with a finite number of sets in the cover. For a locally finite cover of a smooth manifold $\cM$, there always exists a $C^{\infty}$ partition of unity subordinate to the cover \citep[Chapter 2, Theorem 15] {spivak1973comprehensive}. 
\begin{proposition}[Existence of a $C^\infty$ partition of unity]\label{thm:parunity}
	Let $\{U_\alpha\}_{\alpha \in \cA}$ be a locally finite cover of a smooth manifold $\cM$. There is a $C^\infty$ partition of unity $\{\rho_\alpha\}_{\alpha=1}^\infty$ such that $\supp(\rho_\alpha) \subset U_\alpha$. 
\end{proposition}

Let $\{(U_\alpha,\phi_\alpha)\}_{\alpha\in \cA}$ be a $C^\infty$ atlas of $\cM$. Proposition \ref{thm:parunity} guarantees the existence of a partition of unity $\{\rho_\alpha\}_{\alpha\in\cA}$ such that $\rho_{\alpha}$ is supported on $U_{\alpha}$.

%
\section{Convolutional neural networks and multi-layer perceptions}
Our proofs are based on approximation theories of convolutional neural networks (CNN) and their relations to multi-layer perceptions (MLP). In this section, we introduce related notations and definitions. For the convenience of notation, we use $\cdot$ to denote $\otimes$, the sum of entrywise product.

We consider CNNs in the form of
\begin{align}\label{eq:convfCNN}
	f(\xb)=W\cdot\Conv_{\cW,\cB}(\xb),
\end{align}
where $\Conv_{\cW,\cB}(Z)$ is defined in (\ref{eq.conv}), $W$ is the weight matrix of the fully connected layer, $\cW,\cB$ are sets of filters and biases, respectively.
We define the class of CNNs as
\begin{equation*}
	\begin{aligned}
		\cF^{\rm CNN}(L,J,K,\kappa_1,\kappa_2) = \big\{f ~|& f(\xb) \textrm{ in the form \eqref{eq:convfCNN} with $L$ layers.}\\
		&\mbox{Each convolutional layer has filter size bounded by $K$.} \\
		&  \mbox{The number of channels of each layer is bounded  by $J$}.\\
		& \max_{l}\|\cW^{(l)}\|_{\infty} \vee \|B^{(l)}\|_{\infty} \leq \kappa_1,\  \|W\|_{\infty}  \leq \kappa_2\big\}.\label{eqcFCNN}
	\end{aligned}
\end{equation*}

For MLP, we consider the following form
\begin{align}\label{eq:reluf}
	f(\xb) = W_L \cdot \textrm{ReLU}(W_{L-1} \cdots \textrm{ReLU}(W_1 \xb + \bbb_1) \cdots + \bbb_{L-1}) + \bbb_L,
\end{align}
where $W_1, \dots, W_L$ and $\bbb_1, \dots, \bbb_L$ are weight matrices and bias vectors of proper sizes, respectively. 
The class of MLP is defined as
\begin{equation*}
	\begin{aligned}
		\cF^{\rm MLP}(L,J,\kappa) = \big\{f ~|& f(\xb) \textrm{ in the form \eqref{eq:reluf} with $L$-layers and width bounded by $J$}. \\
		&  \norm{W_i}_{\infty, \infty} \leq \kappa, \norm{\bbb_i}_\infty \leq \kappa ~\textrm{for}~ i = 1, \dots, L \big\}.\label{eqcF}
	\end{aligned}
\end{equation*}
In some cases it is necessary to enforce the output of the MLP to be bounded. We define such a class as
\begin{equation*}
	\begin{aligned}
		\cF^{\rm MLP}(L,J,\kappa,R) = \left\{f ~| f(\xb)\in \cF^{\rm MLP}(L,J,\kappa) \mbox{ and } \|f\|_{\infty}\leq R\right\}.
	\end{aligned}
\end{equation*}
In some case we do not need the constraint on the output, we denote such MLP class as $\cF^{\rm MLP}(L,J,\kappa)$.

\section{Proof of Theorem \ref{thm.D}}\label{sec.proof.D}

Before we prove Theorem \ref{thm.D}, we define the Sobolev semi-norm:
\begin{definition}
	For any integers $0\leq k\leq \alpha$, $1\leq p<\infty$ and function $f\in W^{\alpha,p}(\Omega)$, we define its Sobolev semi-norm as
	\begin{align*}
		&|f|_{W^{k,p}(\Omega)}=\Big(\sum_{|\bm{\alpha}|=k} \|D^{\bm{\alpha}}f\|_{L^p(\Omega)}^p\Big)^{1/p},\\
		&|f|_{W^{k,\infty}(\Omega)}=\max_{|\bm{\alpha}|=k} \|D^{\bm{\alpha}}f\|_{L^{\infty}(\Omega) },
	\end{align*}
\end{definition}

Now we prove Theorem \ref{thm.D}.
\begin{proof}[Proof of Theorem \ref{thm.D}]
	We prove Theorem \ref{thm.D} in four steps.
	\paragraph{Step 1: Decompose $(0,1)^D$ using locally supported functions.} We define 
	\begin{align*}
		\psi(x)=\begin{cases}
			1 & |x|<1,\\
			0 & 2<|x|,\\
			2-|x| & 1\leq |x|\leq 2
		\end{cases}
	\end{align*}
	and 
	\begin{align*}
		\phi_{\mb}(\xb)=\prod_{k=1}^D \psi\left(3N\left(x_k-\frac{m_k}{N}\right)\right)
	\end{align*}
	with $\mb=(m_1,m_2,...,m_D)\in \{0,...,N\}^D$. We have $\sum_{\mb} \phi_{\mb}=1$ on $(0,1)^D$ and $\phi_{\mb}$ is supported on $B_{\frac{2}{3}N,\|\cdot\|_{\infty}}(\frac{\mb}{N})\subset B_{1/N,\|\cdot\|_{\infty}}(\frac{\mb}{N})$. We denote $\cS_N=\{0,1,...,N\}^D$. The following lemma shows that each $\psi\left(3N\left(x_k-\frac{m_k}{N}\right)\right)$ can be realized by a CNN (see a proof in Appendix \ref{sec.psi}).
	\begin{lemma}\label{lem.psi}
		There exists a CNN architecture $\cF^{\rm CNN}(L,J,K,\kappa_1,\kappa_2)$ such that for any $N,m$, such an architecture yields a CNN $\widetilde{\psi}$ with
		\begin{align}
			&\widetilde{\psi}_{m,N}(x)=\psi\left(3N\left(x_k-\frac{m}{N}\right)\right), \label{eq.psi.CNN}\\
			&\|\widetilde{\psi}_{m,N}\|_{W^{k,\infty}(0,1)}\leq (3N)^k.
			\label{eq.psi.sobolev}
		\end{align}
		Such an architecture has
		$$
		L=2, \ J=16, \  K=2, \ \kappa_1=\kappa_2=O(N).
		$$
		Further more, the weight matrix in the fully connected layer of $\cF^{\rm CNN}$ has nonzero entries only in the first row.
	\end{lemma}
	We then decompose $f$ as
	\begin{align*}
		f=\sum_{\mb} \phi_{\mb}f.
	\end{align*}
	
	
	\paragraph{Step 2: Approximate each $\phi_{\mb}f$ using averaged Taylor polynomials.}
	On each $B_{1/N,\|\cdot\|_{\infty}}(\frac{\mb}{N})$, we approximate $\phi_{\mb}f$ by an averaged Taylor polynomial. The averaged Taylor polynomial is defined as follows:
	\begin{definition}[Averaged Taylor polynomials]
		Let $\alpha>0, 1\leq p\leq +\infty$ be integers and $f\in W^{\alpha-1,p}(\Omega)$. For $\xb_0\in \Omega, r>0$ such that $\overline{B_{r,\|\cdot\|}(\xb_0)}$ is compact in $\Omega$, the corresponding Taylor polynomial of order $\alpha$ of $f$ averaged over $B_{r,\|\cdot\|}(\xb_0)$ is defined as
		\begin{align*}
			Q_{\xb_0}^{\alpha}f(\xb)=\int_{B_{r,\|\cdot\|}(\xb_0)} T^{\alpha}f(\xb,\zb)\phi(\zb)d\zb
		\end{align*}
		with
		\begin{align*}
			T^{\alpha}f(\xb,\yb)=\sum_{|\vb|\leq \alpha-1} \frac{1}{\vb\!} \partial^{\vb}f(\zb)(\xb-\zb)^{\vb}
		\end{align*}
		and $\phi$ being arbitrary cut-off function satisfying
		\begin{align*}
			&\phi\in C_c^{\infty}(\RR^D) \mbox{ with } \phi(\xb)\geq 0 \mbox{ for all } \xb\in \RR^D, \nonumber\\
			&\supp(\phi)=\overline{B_{r,\|\cdot\|}(\xb_0)} \mbox{ and } \int_{\RR^D}\phi(\xb)d\xb=1,
		\end{align*}
		where $C_c^{\infty}(\RR^D)$ denotes the space of infinitely differentiable functions on $\RR^D$ with compact support.
	\end{definition}
	Under proper assumptions, the averaged Taylor polynomial can approximate $f$ and its partial derivatives well. We first define the star-shaped sets and chunkiness parameter, which are used in the error estimation result.
	\begin{definition}[Star-shaped sets, Definition 4.2.2 of \cite{brenner2008mathematical}]\label{def.starshape}
		Let $\Omega,\widetilde{\Omega}\subset \RR^D$. Then $\Omega$ is called star-shaped with respect to $\widetilde{\Omega}$ if for all $\xb\in\Omega$, we have
		\begin{align*}
			\overline{\convex(\{\xb\}\cup \widetilde{\Omega})}\subset \Omega.
		\end{align*}
	\end{definition}
	
	\begin{definition}[Chunkiness parameter, Definition 4.2.16 of \cite{brenner2008mathematical}]\label{def.chunkiness}
		Let $\Omega\subset\RR^D$ be bounded. Define
		\begin{align*}
			\cR=\big\{ &r>0: \mbox{ there exists } \xb\in \Omega \mbox{ such that } \Omega \mbox{ is star-shaped with respect to } B_{r,|\cdot|}(\xb) \big\}.
		\end{align*}
		
		For $\cR\neq \emptyset$, we define
		$$
		r_{\max}^*=\sup \cR \quad  \mbox{ and } \quad \gamma=\frac{\diam(\Omega)}{r_{\max}^*},
		$$
		where $\gamma$ is called the chunkiness parameter of $\Omega$.
	\end{definition}
	The following lemma gives an error estimation of averaged Taylor polynomials:
	\begin{lemma}[Bramble-Hilbert, Lemma 4.3.8 of \cite{brenner2008mathematical}]
		\label{lem.ATE.error}
		Let $\Omega\subset \RR^D$ be open and bounded, $\xb\in \Omega$ and $r>0$ such that $\Omega$ is star-shaped with respect to $B_{r,\|\cdot\|}(\xb_0)$ and $r>\frac{1}{2} r^*_{\max}$, with $r^*_{\max}$ defined in Definition \ref{def.chunkiness}. Let $n>0,1\leq p \leq +\infty$ be integers and $\gamma$ be the chunkiness parameter of $\Omega$. Then we have
		\begin{align*}
			|f-Q^{\alpha}_{\xb_0}f|_{W^{\alpha,p}(\Omega)}\leq Ch^{\alpha-k}|f|_{W^{\alpha,p}(\Omega)}
		\end{align*}
		for $k=0,1,...,\alpha$, where $h=\diam(\Omega)$ and $C$ is a constant depending on $D,\alpha,\gamma$.
	\end{lemma}
	Lemma \ref{lem.aveTaylor.err} below shows that $Q^{\alpha}f$ can be written as a weighted sum of polynomials.
	\begin{lemma}[Lemma B.9 of \cite{guhring2020error}]\label{lem.aveTaylor.err}
		Let $\alpha>0, 1\leq p\leq +\infty$ be integers and $f\in W^{\alpha-1,p}(\Omega)$. Let $\xb_0\in\Omega, r>0$ such that $\overline{B_{r,\|\cdot\|}(\xb_0)}$ is compact in $\Omega$, and there exists $\widetilde{r}>0$ with  $B_{r,\|\cdot\|}(\xb_0)\subset B_{\widetilde{r},\|\cdot\|_{\infty}}(0)$. Then the averaged Taylor polynomial $Q_{\xb_0}^{\alpha}(f)$ can be written as
		\begin{align}
			Q_{\xb_0}^{\alpha}f(\xb)=\sum_{|\vb|\leq \alpha-1} c_{\vb}\xb^{\vb}
			\label{eq.aveTayloer.poly}
		\end{align}
		for $\xb\in \Omega$. There exists a constant $C$ depending on $\alpha,D,\widetilde{r}$ such that 
		$$
		|c_{\vb}|\leq Cr^{-D/p}\|f\|_{W^{\alpha-1,p}}(\Omega)
		$$
		for all $|\vb|\leq \alpha-1$.
	\end{lemma}
	
	Using averaged Taylor polynomials, we approximate $\phi_{\mb}f$ by
	\begin{align}
		\phi_{\mb}f\approx (\phi_{\mb}Q_{\mb/N}^{\alpha}f)(\xb)&=\phi_{\mb}\sum_{|\vb|\leq \alpha-1} c_{\mb,\vb}\xb^{\vb} =\sum_{|\vb|\leq \alpha-1} c_{\mb,\vb}\phi_{\mb}\xb^{\vb}.	
		\label{eq.f.aveTaylor}
	\end{align}
	Define 
	\begin{align}
		\widehat{f}=\sum_{\mb\in S_N}\sum_{|\vb|\leq \alpha-1} c_{\mb,\vb}\phi_{\mb}\xb^{\vb},
		\label{eq.f.aveTaylor1}
	\end{align}
	where $c_{\mb,\vb}$'s are the coefficients in (\ref{eq.aveTayloer.poly}).
	Then $\widehat{f}$ is an approximation of $f$. The following lemma gives an upper bound on the approximation error
	\begin{lemma}[Lemma C.4 of \cite{guhring2020error}]
		\label{lem.ATE.approx.error}
		Let $\alpha\geq 2$ be an integer and $1\leq p\leq \infty$. For any $s\in[0,1]$ and $f\in W^{\alpha,p}((0,1)^D)$, one has
		\begin{align*}
			\|\widehat{f}-f\|_{W^{s,p}((0,1)^D)}\leq C\left( \frac{1}{N}\right)^{\alpha-s}\|f\|_{W^{\alpha,p}((0,1)^D)},
		\end{align*}
		where $C$ is a constant depending on $\alpha,p,D$. Furthermore, the coefficients in $\widehat{f}$ satisfies
		\begin{align*}
			|c_{\mb,\vb}|\leq C_1N^{D/p}\|f\|_{W^{\alpha,p}((0,1)^D)}
		\end{align*} 
		for some constant $C_1$ depending on $D,\alpha,p$.
	\end{lemma}
	
	\paragraph{Step 3: Network approximation}
	Note that $\widehat{f}$ is a sum of functions in the form of $\phi_{\mb}\xb^{\vb}$ with weights $c_{\mb,\vb}$'s. We next approximate each $\phi_{\mb}\xb^{\vb}$ by a CNN.
	
	\begin{lemma}\label{lem.product}
		For any $0<\varepsilon<1, \xb\in (0,1)^D, N>0, \mb\in \{0,1,...,N\}^D,|\vb|<\alpha$, there exists a CNN architecture $\cF^{\rm CNN}(L,J,K,\kappa,\kappa)$ that yields a CNN $\widetilde{g}$ with
		\begin{align}
			&\|\widetilde{g}_{\mb,\vb}(\xb)-\phi_{\mb}\xb^{\vb}\|_{W^{k,\infty}((0,1)^D}\leq C_2N^k\varepsilon,
			\label{eq.product}\\
			&\widetilde{g}_{\mb,\vb}(\xb)=0 \mbox{ if } \phi_{\mb}\xb^{\vb}=0
			\label{eq.product0}
		\end{align}
		for $k=0,1$, where $C_2$ is a constant depending on $\alpha,k$. Such an architecture has
		\begin{align*}
			L=O\left(D\log \frac{1}{\varepsilon}\right),\ J=O(D),\ \kappa=3N.
		\end{align*}
		The constants hidden in $O$ depends on $\alpha,k$. 	Further more, the weight matrix in the fully connected layer of $\cF^{\rm CNN}$ has nonzero entries only in the first row.
	\end{lemma}
	Lemma \ref{lem.product} is proved in Appendix \ref{sec.product}. By Lemma \ref{lem.product}, each $\phi_{\mb}\xb^{\vb}$ can be approximated by a CNN. Denote the network approximation of $\phi_{\mb}\xb^{\vb}$ by $\widetilde{g}_{\mb,\vb}(\xb)$. We approximate $\widehat{f}$ by $\widetilde{f}$ defined as
	\begin{align}
		\widetilde{f}=\sum_{\mb} \sum_{|\vb|\leq \alpha-1} c_{\mb,\vb}\widetilde{g}_{\mb,\vb}(\xb).
		\label{eq.ftilde}
	\end{align}

	The following lemma gives an upper bound of the approximation error of $\widetilde{f}$ (see a proof in Appendix \ref{sec.net.error}).
	\begin{lemma}
		\label{lem.net.error}
		Let $\alpha\geq 2$ and $1\leq p\leq \infty$ be integers.
		For any $f\in W^{\alpha,p}((0,1)^D)$, let $\phi_{\mb}Q_{\mb/N}^{\alpha}f(\xb)$ be the averaged Taylor approximation of $\phi_{\mb}f$ defined in (\ref{eq.f.aveTaylor}). For any $0<\eta<1$, let $\widetilde{g}_{\mb,\vb}$ be the CNN approximation of $\phi_{\mb}Q_{\mb/N}^{\alpha}f(\xb)$ constructed in Lemma \ref{lem.product} with accuracy $\eta$.
		For $0\leq s \leq 1$, we have
		\begin{align}
			&\|\sum_{\mb\in \cS_N} \phi_{\mb}Q_{\mb/N}^{\alpha}f-\sum_{\mb\in \cS_N} \sum_{|\vb|\leq \alpha-1}c_{\mb,\vb}\widetilde{g}_{\mb,\vb}\|_{W^{s,p}((0,1)^D)} \leq C_3\|f\|_{W^{\alpha,p}((0,1)^D)}N^s\eta,
			\label{eq.neterr}
		\end{align}
		where $c_{\mb,\vb}$'s are coefficients defined in (\ref{eq.f.aveTaylor}), $C_3$ is a constant depending on $D,\alpha,s,p$.
	\end{lemma}
	Note the $\widetilde{f}$ is the sum of no more than $N^D(D+1)^{\alpha-1}$ CNNs of which the width is of $J=O(D)$. The following lemma shows that under appropriate conditions, the sum of $n_0$ CNNs with width in the same order can be realized by the sum of $n_1$ CNNs with a proper width (see a proof in Appendix \ref{sec.CNN.adap}):
	\begin{lemma}\label{lem.CNN.adap}
		Let $\{f_i\}_{i=1}^{n_0}$ be a set of CNNs with architecture $\cF^{\rm CNN}(L_0,J_0,K_0,\kappa_0,\kappa_0)$. For any integers $1\leq n\leq n_0$ and $\widetilde{J}$ satisfying $n\widetilde{J}=O(n_0J_0)$ and $\widetilde{J}\geq J_0$, there exists a CNN architecture $\cF^{\rm CNN}(L,J,K,\kappa,\kappa)$ that gives a set of CNNs $\{g_i\}_{i=1}^{n}$  such that 
		\begin{align*}
			\sum_{i=1}^{n} g_i(\xb)= \sum_{i=1}^{n_0} f_i(\xb).
		\end{align*}
		Such an architecture has 
		\begin{align*}
			L=O(L_0), J=O(\widetilde{J}), K=K_0, \kappa=\kappa_0.
		\end{align*}
		Furthermore, the fully connected layer of $f$ has nonzero elements only in the first row.
	\end{lemma}

	By Lemma \ref{lem.CNN.adap}, for any $\widetilde{M},\widetilde{J}$ satisfying $\widetilde{M}\widetilde{J}=O(N^D)$, there exists a CNN architecture $\cF^{\rm CNN}(L,J,K,\kappa,\kappa)$ that gives rise to $\{g_i\}_{i=1}^n$ with 
	\begin{align*}
		\widetilde{f}=\sum_{i=1}^{\widetilde{M}} g_i,
	\end{align*}
	where
	\begin{align*}
		L=O\left(\log \frac{1}{\eta}\right),\ J=O(\widetilde{J}),\ \kappa=3N.
	\end{align*}
	
	The following lemma shows that the sum of CNNs can be realized by a ConvResNet:
	\begin{lemma}[Lemma 18 in \cite{liu2021besov}]\label{lem.cnn.convresnet}
		Let $\cF^{\rm CNN}(L,J,K,\kappa_1,\kappa_2)$ be any CNN architecture from $\RR^D$ to $\RR$. Assume the weight matrix in the fully connected layer of $\cF^{\rm CNN}(L,J,K,\kappa_1,\kappa_2)$ has nonzero entries only in the first row. Let $M$ be a positive integer. There exists a ConvResNet architecture $\cC(M,L,J,\kappa_1, \kappa_2(1\vee \kappa_1^{-1}))$ such that for any $\{f_i(\xb)\}_{i=1}^M\subset\cF^{\rm CNN}(L,J,K,\kappa_1,\kappa_2) $, there exists $\widetilde{f}\in \cC(M,L,J,\kappa_1, \kappa_2(1\vee \kappa_1^{-1}))$ with
		$$
		\widetilde{f}(\xb)=\sum_{i=1}^M f_i(\xb).
		$$
	\end{lemma}
	By Lemma \ref{lem.cnn.convresnet}, there exits a ConvResNet architecture  $\cC(M,L,J,K,\kappa_1,\kappa_2)$ with
	\begin{align}
		&L=O(\log 1/\eta), \ J=O(\widetilde{J}), \ \kappa_1=O(3N), \ \kappa_2=O(3N), \ M=O(\widetilde{M})
		\label{eq.net.para.1}
	\end{align}
	and $\widetilde{J},\widetilde{M}$ satisfying
	\begin{align}
		\widetilde{M}\widetilde{J}=O(N^D),
		\label{eq.net.para.2}
	\end{align}
	that yields a ConvResNet realizing $\widetilde{f}$.

	\paragraph{Step 4: Error estimation.} We compute
	\begin{align}
		&\|f-\widetilde{f}\|_{W^{s,p}((0,1)^D)} \nonumber\\
		\leq &\left\|f-\left(\sum_{\mb} \phi_{\mb}Q_{\mb/N}^{\alpha}f\right)\right\|_{W^{s,p}((0,1)^D)} +\left\|\left(\sum_{\mb} \phi_{\mb}Q_{\mb/N}^{\alpha}f\right)-\widetilde{f}\right\|_{W^{s,p}((0,1)^D)} \nonumber\\
		\leq &C_4\left( \frac{1}{N}\right)^{\alpha-s}\|f\|_{W^{\alpha,p}((0,1)^D)}+ C_5N^{s}\eta \|f\|_{W^{\alpha,p}((0,1)^D)} \nonumber\\
		\leq & (C_4+C_5)N^{-(\alpha-s)},
		\label{eq.neterror}
	\end{align}
	where $C_4,C_5$ are two constants depending on $D,\alpha,s,p,R$. In the second inequality, we use Lemma \ref{lem.ATE.approx.error} and \ref{lem.net.error} for the first and second term, respectively. In the third inequality, we set $\eta=N^{-\alpha}$ to balance the two terms. Using the relation (\ref{eq.net.para.2}), we have 
	\begin{align}
		N=(\widetilde{M}\widetilde{J})^{1/D}, \quad \eta=(\widetilde{M}\widetilde{J})^{-\frac{\alpha}{D}}.
		\label{eq.net.para.3}
	\end{align}
	
	Substituting (\ref{eq.net.para.3}) into (\ref{eq.neterror}) gives rise to 
	\begin{align}
		\|f-\widetilde{f}\|_{W^{s,p}((0,1)^D)}\leq C_6(\widetilde{M}\widetilde{J})^{-\frac{\alpha-s}{D}}
		\label{eq.neterror.1}
	\end{align}
	for some constant $C_6$ depending on $D,\alpha,s,p,R$.
	Substituting (\ref{eq.net.para.3}) into (\ref{eq.net.para.1}) and (\ref{eq.net.para.2}) gives rise to the network architecture
	\begin{align*}
		&L=O(\log (\widetilde{M}\widetilde{J})), \ J=O(\widetilde{J}), \ \kappa_1=O((\widetilde{M}\widetilde{J})^{1/D}), \ \kappa_2=O((\widetilde{M}\widetilde{J})^{1/D}), \ M=O(\widetilde{M}).
	\end{align*}
\end{proof}

\section{Proof of Theorem \ref{thm.prob}}
\label{sec.proof.prob}
\begin{proof}[Proof of Theorem \ref{thm.prob}]
	By Theorem \ref{thm.D} and the choice of $\widetilde{M}\widetilde{J}$, there exits $\widetilde{f}\in \cC$ so that $\|\widetilde{f}-f\|_{\infty}\leq \varepsilon$ and 
	\begin{align}
		\max_j\left\|\frac{\partial \widetilde{f}}{\partial x_j}-\frac{\partial f}{\partial x_j}\right\|\leq \varepsilon^{\frac{\alpha-1}{\alpha}},
	\end{align}
	which implies
	\begin{align}
		\left\|\widetilde{f}\right\|_{\rm Lip}\leq 1+\sqrt{D}\varepsilon^{\frac{\alpha-1}{\alpha}}.
	\end{align}
	We have
	\begin{align}\textbf{}
		&\EE\left[(\widetilde{f}(\xb_1)-y_1)^2\right] \nonumber\\
		\leq &\EE\left[(\widetilde{f}(\xb_1)-f(\xb_1))^2\right] +\EE\left[(f(\xb_1)-y_1)^2\right] \nonumber\\
		\leq & \varepsilon^2 +\sigma^2.
	\end{align}
	Denote $X_i=\frac{1}{n} (\widetilde{f}(\xb_i)-y_i)^2-\EE\left[(\widetilde{f}(\xb_i)-y_i)^2\right]$. We have 
	\begin{align}
		|X_i|\leq \frac{2(\varepsilon^2+\sigma^2)}{n}, \ \EE[X_i]=0,
	\end{align}
	and 
	\begin{align}
		\EE[X_i^2]\leq \frac{8(\varepsilon^4+\sigma^4)}{n^2}.
	\end{align}
	By Bernstein inequality, we deduce
	\begin{align}
		\PP\left(\sum_{i=1}^n X_i\geq t\right)\leq &\exp\left( -\frac{\frac{1}{2}t^2}{\frac{8(\varepsilon^4+\sigma^4)}{n}+\frac{2(\varepsilon^2+\sigma^2)}{3n}t} \right) \nonumber\\
		=& \exp\left( -\frac{3nt^2}{48(\varepsilon^4+\sigma^4)+4(\varepsilon^2+\sigma^2)t} \right).
	\end{align} 
	Therefore
	\begin{align}
		&\PP\left(\frac{1}{n} \sum_{i=1}^n (\widetilde{f}(\xb_i)-y_i)^2\geq \varepsilon^2+\sigma^2+t \right) \nonumber\\
		\leq & \PP\left(\frac{1}{n} \sum_{i=1}^n (\widetilde{f}(\xb_i)-y_i)^2\geq \EE\left[(\widetilde{f}(\xb_1)-y_1)^2\right]+t \right) \nonumber\\
		\leq & \exp\left( -\frac{3nt^2}{48(\varepsilon^4+\sigma^4)+4(\varepsilon^2+\sigma^2)t} \right).
	\end{align}
	Setting $t=\varepsilon^2$ gives rise to
	\begin{align}
		&\PP\left(\frac{1}{n} \sum_{i=1}^n (\widetilde{f}(\xb_i)-y_i)^2\geq 2\varepsilon^2+\sigma^2 \right) \nonumber\\
		\leq & \exp\left( -\frac{3n\varepsilon^2}{104\sigma^4} \right).
	\end{align}
\end{proof}

\section{Proof of Theorem \ref{thm.adversarial}}
\label{sec.proof.adversarial}
\begin{proof}[Proof of Theorem \ref{thm.adversarial}]
	By Theorem \ref{thm.D} and the choice of $\widetilde{M}\widetilde{J}$, there exits $\widetilde{f}\in \cC$ so that $\|\widetilde{f}-f\|_{\infty}\leq \varepsilon$ and 
	\begin{align}
		\max_j\left\|\frac{\partial \widetilde{f}}{\partial x_j}-\frac{\partial f}{\partial x_j}\right\|\leq \varepsilon^{\frac{\alpha-1}{\alpha}}.
	\end{align}
	Since $\|f\|_{W^{\alpha,\infty}}\leq 1$, we have 
	\begin{align}
		\| \widetilde{f}\|_{\rm Lip}\leq 1+\sqrt{D}\varepsilon^{\frac{\alpha-1}{\alpha}}.
	\end{align}
	We have
	\begin{align}
		&R(\widetilde{f},\delta)-R(\widetilde{f},0) \nonumber\\
		=&\EE_{(\xb,y)\in \supp(\rho)} \left[\sup_{\xb'\in B_{\delta}(\xb)} \ell\left(\widetilde{f}(\xb'),y\right)\right]-\EE_{(\xb,y)\in \supp(\rho)} \left[\ell\left(\widetilde{f}(\xb'),y\right)\right]\nonumber\\
		\leq & \EE_{(\xb,y)\in \supp(\rho)} \left[\sup_{\xb'\in B_{\delta}(\xb)} \left|\ell\left(\widetilde{f}(\xb'),y\right)-\ell\left(\widetilde{f}(\xb),y\right)\right| \right]\nonumber\\
		\leq& \EE_{(\xb,y)\in \supp(\rho)} \sup_{\xb'\in B_{\delta}(\xb)} L_{\rm Lip}|\widetilde{f}(\xb')-\widetilde{f}(\xb)|\nonumber\\
		\leq &\EE_{(\xb,y)\in \supp(\rho)} \sup_{\xb'\in B_{\delta}(\xb)} L_{\rm Lip}\|\widetilde{f}\|_{\rm Lip}\|\xb'-\xb\|_2\nonumber\\
		\leq & L_{\rm Lip}(1+\sqrt{D}\varepsilon^{\frac{\alpha-1}{\alpha}})\delta
	\end{align}
\end{proof}

\section{Proof of Theorem \ref{thm.M}}
\label{sec.proof.M}

\begin{proof}[Proof of Theorem \ref{thm.M}]
	We prove Theorem \ref{thm.M} in three steps. \\
	\textbf{Step 1: Decomposition of $f$}\\
	$\bullet$ \textbf{Construct an atlas on $\cM$.} According to Assumption \ref{assum.M}, $\cM$ is bounded. Therefore, for any given $0<r<\tau/2$, we can find a finite collection of points $\{\cbb_i\}_{i=1}^{C_{\cM}}\subset \cM$ such that 
	$$\cM\subset \bigcup_{i=1}^{C_{\cM}}B_r(\cbb_i). $$ 
	Denote $U_i=B_r(\cbb_i)\cap \cM$. Then $\{U_i\}_{i=1}^{C_{\cM}}$ form an open cover of $\cM$ and each $U_i$ is diffeomorphic to an open subset of $\RR^d$. The total number of partitions if bounded by $C_{\cM}\leq \left\lceil \frac{\mathrm{SA}(\cM)}{r^d}T_d\right\rceil$, where $\mathrm{SA}(\cM)$ is the surface area of $\cM$ and $T_d$ is the average number of $U_i$'s that contain a given point on $\cM$. 
	
	On each $U_i$, we define a transformation $\phi_i$ that projects any $\xb\in U_i$ to $T_{\cbb_i}(\cM)$, the tangent space of $\cM$ at $\cbb_i$. Let $V_i\in \RR^{D\times d}$ be an orthogonal matrix whose columns form an orthonomal basis of $T_{\cbb_i}(\cM)$. Define 
	\begin{align}
		\varphi_i(\xb)=a_iV_i^{\top}(\xb-\cbb_i)+\bb_i \mbox{ for } \xb\in U_i,
		\label{eq.transformation}
	\end{align}
	where $a_i\in \RR$ is a scaling factor and $\bb)i\in \RR^d$ is a shifting vector that ensure $\varphi_i(U_i)\subseteq [0,1]^d$. Then $\{(U_i,\varphi_i)\}_{i=1}^{C_{\cM}}$ form an atlas of $\cM$.
	
	\noindent$\bullet$\textbf{Decomposition of $f$ by a partition of unity.}
	The following lemma shows that under proper assumption, there exists a partition of unity $\{\rho_i\}_{i=1}^{C_{\cM}}$ subordinate to  $\{(U_i,\varphi_i)\}_{i=1}^{C_{\cM}}$ (see Appendix \ref{sec.proof.parunity} for a proof).
	\begin{lemma}\label{lem.parunity}
		Let $\{(U_i,\varphi_i)\}_{i=1}^{C_{\cM}}$ be the atlas of $\cM$ defined above with $r<\tau/4$. There exist a finite number $C_{\cM}$ and a $C^{\infty}$ partition of unity $\{\rho_i\}_{i=1}^{C_{\cM}}$ satisfying
		\begin{enumerate}
			\item[(i)] $\supp(\rho_i)$ is compact in $U_i$.
			\item[(ii)] $\sum_{i=1}^{C_{\cM}} \rho_i(\xb)=1 $ for any $\xb\in\cM$.
			\item[(iii)] There exists a constant $c>0$ depending on $r$ such that for any $i$, we have 
			\begin{align*}
				\inf_{\xb\in \supp(\rho_i),\ \widetilde{\xb}\in \partial U_i} \|\xb-\widetilde{\xb}\|_{2}\geq c.
			\end{align*}
		\end{enumerate}
		Here $C_{\cM}$ depends on the surface area of $\cM$ and the average number of $U_i$'s that contain a given point on $\cM$.
	\end{lemma}

	Let $\{\rho_i\}_{i=1}^{C_{\cM}}$ be the partition of unity from Lemma \ref{lem.parunity}. Since for each $i$, $\varphi_i$ is a bijection from $U_i$ to a subset of $[0,1]^d$, $\varphi^{-1}$ exists and is a linear operator. We decompose $f$ as
	\begin{align*}
		f=\sum_{i=1}^{C_{\cM}} f_i \quad \mbox{ with } \quad f_i=(f\rho_i).
	\end{align*}
	Here each $f_i$ is compactly supported on $U_i$ and each $f_i\circ \varphi_i^{-1}$ is compactly supportedin $\varphi_i(U_i)\subseteq [0,1]^d$. We extend $f_i\circ \varphi_i^{-1}$ by $0$ on $[0,1]^d\backslash \varphi_i(U_i)$. The extended function is in $W^{\alpha,k}([0,1]^d)$. To simplify the notation, we still use $f_i\circ \varphi_i^{-1}$ to denote the extended function. For each $i$, we use averaged Taylor polynomials to approximate $f_i\circ\varphi^{-1}$ on $[0,1]^d$ as in (\ref{eq.f.aveTaylor1}):
	\begin{align*}
		f_i\circ\varphi_i^{-1}\approx \widehat{f}_i=\sum_{\mb,\vb} c_{i,\mb,\vb}\phi_{\mb}\xb^{\vb}.
	\end{align*}
	\noindent\textbf{Step 2: Network approximation}\\
	\noindent$\bullet$\textbf{Approximate $\widehat{f}_i$ by CNNs.} Since each $\widehat{f}_i$ is the averaged Taylor polynomial approximation of $f_i\circ\varphi_i^{-1}$, by Lemma \ref{lem.product}, it can be approximated by a sum of $(d+1)^{\alpha-1}N^d$ CNNs. Denote the approximation accuracy by $\eta$ as in Lemma \ref{lem.product}, each CNN has depth $O(\log (1/\eta))$, width $O(1)$, all weight parameters are of $O(N)$.
	
	\noindent$\bullet$\textbf{Chart determination}
	For any input $\xb$, to determine the chart it belongs to, we are going to construct an indicator function. With our construction of charts, we have $\xb\in U_i$ if and only if $\|\xb-\cbb_i\|_2^2\leq r^2$. Define the indicator function
	\begin{align*}
		\mone_{[0,r^2]}(a)=\begin{cases}
			1 & \mbox{ if } a\leq r^2,\\
			0 & \mbox{ otherwise,}
		\end{cases}
	\end{align*}
	and the	squared distance function
	\begin{align}
		d_i^2(\xb)=\|\xb-\cbb_i\|_2^2=\sum_{j=1}^D (x_j-c_{i,j})^2,
		\label{eq.dis}
	\end{align}
	where we used the expression $\xb=[x_1,...,x_D]^{\top}$ and $\cbb_i=[c_{i,1},...,c_{i,D}]^{\top}$. The composition $\mone_i=\mone_{[0,r^2]}\circ d_i^2$ outputs $1$ if $\xb\in U_i$ and outputs $0$ otherwise. We are going to construct a CNN to approximate $\mone_i$.
	
	 In (\ref{eq.dis}), the function $d_i^2$ is a sum of $D$ square functions. By Lemma \ref{lem.multiplication.CNN}, For any $0<\theta<1/2$, $x\in[-B,B]$, and $K\geq 2$, there is a CNN architecture $\cF^{\rm CNN}(L,J,K,\kappa,\kappa)$ that yields a CNN, denoted by $\widetilde{d}^2$, such that
	$$
	\|\widetilde{d}^2(x)-x^2\|_{W^{1,\infty}([-B,B])}<\theta,\ \widetilde{d}(0)=0.
	$$
	Such a network has 
	$$
	L=O\left(\log \frac{1}{\theta}\right),\ J=24,\ \kappa=1.
	$$
	Furthermore, one has
	\begin{align}
		\|\widetilde{d}^2\|_{W^{1,\infty}((-B,B))}\leq C_7B
		\label{eq.square.partial}
	\end{align}
	for some absolute constant $C_7$. We approximate $d_i$ by
	\begin{align*}
		\widetilde{d}^2_i(\xb)=\sum_{j=1}^D \widetilde{d}^2(x_j-c_{i,j}).
	\end{align*}
	According to Lemma \ref{lem.CNN.adap}, $\widetilde{d}_i$ can be realized by a CNN with $O\left(\log \frac{1}{\theta}\right)$ layers, $O(D)$ width and all weight parameters of $O(1)$. The approximation error is bounded as
	\begin{align*}
		\|\widetilde{d}_i^2-d_i^2\|_{L^{\infty}}\leq 4B^2D\theta.
	\end{align*}

	The following Lemma shows that $\mone_{[0,r^2]}$ can be approximated by a CNN:
	\begin{lemma}[Lemma 9 of \citet{liu2021besov}]\label{lem.M.indicator}
		For any $0<\theta<1$ and $\Delta \geq 8B^2D\theta$, there exists a CNN $\mtoned$ approximating $\mone_{[0,\omega^2]}$ with
		\begin{align*}
			&\mtoned(\xb)=\begin{cases}
				1,&\mbox{ if } a\leq(1-2^{-w})(r^2-4B^2D\theta),\\
				0,&\mbox{ if } a\geq r^2-4B^2D\theta,\\
				2^w((r^2-4B^2D\theta)^{-1}a-1),&\mbox{ otherwise}
			\end{cases}
		\end{align*}
		for $\xb\in\cM$, where $w=\left\lceil\log(r^2/\Delta)\right\rceil$ such that $(1-2^{-k})(\omega^2-4B^2D\theta)\geq \omega^2-\Delta +4B^2D\theta$. Such a CNN has $\left\lceil\log(r^2/\Delta)\right\rceil+D$ layers, $2$ channels. All weight parameters are of $O(1)$.
	\end{lemma}
	Let $\mone_{\Delta}$ be the CNN defined in Lemma \ref{lem.M.indicator}. We have 
	\begin{align}
		\frac{\partial 	\widetilde{\mone}_{\Delta}(a)}{\partial a}=\begin{cases}
			0, &\mbox{ if } a\leq (1-2^{-w})(r^2-4B^2D\theta) \mbox{ or } a\geq r^2-4B^2D\theta,\\
			C_8/\Delta, & \mbox{ otherwise }
		\end{cases}
		\label{eq.indi.partial}
	\end{align}
	for some constant $C_8$ depending on $r$.

	The function $\mone_i$ is approximated by
	\begin{align*}
		\widetilde{\mone}_i(\xb)=\widetilde{\mone}_{\Delta}\circ \widetilde{d}^2_i(\xb).
	\end{align*}
	Combining (\ref{eq.indi.partial}) and (\ref{eq.square.partial}) gives rise to
	\begin{align*}
		\left|\frac{\partial \widetilde{\mone}_i}{\partial x_j}\right|=\left|\left.\frac{\partial \widetilde{\mone}_{\Delta}}{a}\right|_{\widetilde{d}^2_i(\xb)}\right|\left| \frac{\widetilde{d}_i}{\partial x_j}\right|\leq \begin{cases}
			0, &\mbox{ if } d_i(\xb)^2\geq r^2 \mbox{ or } d_i^2(\xb)\leq r^2-\Delta,\\
			CB/\Delta, &\mbox{ otherwise}.
		\end{cases}
	\end{align*}

	\noindent\textbf{Step 3: Error analysis.} Our network approximation of $f$ is 
	\begin{align}
		\widetilde{f}=\sum_{i=1}^{C_{\cM}}\widetilde{f}_i \quad \mbox{ with } \quad \widetilde{f}_i(\xb)=\sum_{\mb,\vb} c_{i,\mb,\vb}(\widetilde{g}_{\mb,\vb}\circ\varphi_i(\xb))\widetilde{\times} \widetilde{\mone}_i(\xb),
		\label{eq.tildef.1}
	\end{align}
	where $\widetilde{g}_{\mb,\vb}$ is the CNN approximation of $\phi_{\mb}\zb^{\vb}$ for $\zb\in[0,1]^d$ as in (\ref{eq.ftilde}).
	We decompose the error as
	\begin{align}
		\|\widetilde{f}-f\|_{W^{k,\infty}(\cM)}\leq &\sum_{i=1}^{C_{\cM}} \|\widetilde{f}_i-f_i\|_{W^{k,\infty}(U_i)} \nonumber\\
		=&\sum_{i=1}^{C_{\cM}} \|\widetilde{f}_i\circ\varphi_i^{-1}\circ\varphi_i-f_i\circ\varphi_i^{-1}\circ\varphi_i\|_{W^{k,\infty}(U_i)}\nonumber\\
		\leq & \sum_{i=1}^{C_{\cM}} \|\widetilde{f}_i\circ\varphi_i^{-1}(\zb)-f_i\circ\varphi_i^{-1}(\zb)\|_{W^{k,\infty}(\varphi_i(U_i))} \tag{set $\zb=\varphi_i(\xb)$}\\
		\leq & \sum_{i=1}^{C_{\cM}} \|\widetilde{f}_i\circ\varphi_i^{-1}(\zb)-f_i\circ\varphi_i^{-1}(\zb)\|_{W^{k,\infty}(\varphi_i(U_i))} \nonumber\\
		\leq & \sum_{i=1}^{C_{\cM}} \|\widetilde{f}_i\circ\varphi_i^{-1}(\zb)-\widehat{f}_i(\zb)\|_{W^{k,\infty}(\varphi_i(U_i))} + \|\widehat{f}_i(\zb)-f_i\circ\varphi_i^{-1}(\zb)\|_{W^{k,\infty}([0,1]^d)}.
		\label{eq.M.err}
	\end{align}
	
	The second term can be bounded using Lemma \ref{lem.ATE.approx.error}. We next focus on the first term
	\begin{align}
		&\|\widetilde{f}_i\circ\varphi_i^{-1}(\zb)-\widehat{f}_i(\zb)\|_{W^{k,\infty}(\varphi_i(U_i))} \nonumber\\
		\leq &  \left\|\sum_{\mb,\vb}c_{i,\mb,\vb}\left[(\widetilde{g}_{\mb,\vb}(\zb))\widetilde{\times} (\widetilde{\mone}_i\circ\varphi_i^{-1}(\zb))-\phi_{\mb}(\zb)\zb^{\vb}\right]\right\|_{W^{k,\infty}(\varphi_i(U_i))} \nonumber\\
		\leq& \left\|\sum_{\mb,\vb}c_{i,\mb,\vb}\left[(\widetilde{g}_{\mb,\vb}(\zb))\widetilde{\times} (\widetilde{\mone}_i\circ\varphi_i^{-1}(\zb))-(\widetilde{g}_{\mb,\vb}(\zb))\times (\widetilde{\mone}_i\circ\varphi_i^{-1}(\zb))\right]\right\|_{W^{k,\infty}(\varphi_i(U_i))} \nonumber\\
		&+\left\|\sum_{\mb,\vb}c_{i,\mb,\vb}\left[(\widetilde{g}_{\mb,\vb}(\zb))\times (\widetilde{\mone}_i\circ\varphi_i^{-1}(\zb))-(\widetilde{g}_{\mb,\vb}(\zb))\times ({\mone}_i\circ\varphi_i^{-1}(\zb))\right]\right\|_{W^{k,\infty}(\varphi_i(U_i))} \nonumber\\
		&+\left\|\sum_{\mb,\vb}c_{i,\mb,\vb}\left[(\widetilde{g}_{\mb,\vb}(\zb))\times ({\mone}_i\circ\varphi_i^{-1}(\zb))-\phi_{\mb}(\zb)\zb^{\vb}\right]\right\|_{W^{k,\infty}(\varphi_i(U_i))} \nonumber\\
		=& \left\|A_1\right\|_{W^{k,\infty}(\varphi_i(U_i))} 
		+\left\|A_2\right\|_{W^{k,\infty}(\varphi_i(U_i))}
		+\left\|A_3\right\|_{W^{k,\infty}(\varphi_i(U_i))}
		\label{eq.M.net.err}
	\end{align}
	with
	\begin{align}
		&A_1=\sum_{\mb,\vb}c_{i,\mb,\vb}\left[(\widetilde{g}_{\mb,\vb}(\zb))\widetilde{\times} (\widetilde{\mone}_i\circ\varphi_i^{-1}(\zb))-(\widetilde{g}_{\mb,\vb}(\zb))\times (\widetilde{\mone}_i\circ\varphi_i^{-1}(\zb))\right], \label{eq.M.A1}
		\\
		&A_2=\sum_{\mb,\vb}c_{i,\mb,\vb}\left[(\widetilde{g}_{\mb,\vb}(\zb))\times (\widetilde{\mone}_i\circ\varphi_i^{-1}(\zb))-(\widetilde{g}_{\mb,\vb}(\zb))\times ({\mone}_i\circ\varphi_i^{-1}(\zb))\right], \label{eq.M.A2}\\
		&A_3=\sum_{\mb,\vb}c_{i,\mb,\vb}\left[\widetilde{g}_{\mb,\vb}(\zb)-\phi_{\mb}(\zb)\zb^{\vb}\right].\label{eq.M.A3}
	\end{align}
	Denote the $W^{1,\infty}$ error of $\widetilde{\times}$ by $\delta$.
	We first derive an upper bound for $A_1$. We can show that $\|\widetilde{g}_{\mb,\vb}\|_{\infty}\leq \alpha+d$ (see (\ref{eq.gt.W0})) and $\|\widetilde{\mone}_i\circ\varphi_i^{-1}\|_{L^{\infty}}=1$. Therefore by Lemma \ref{lem.multiplication}, we have for $k=0$
	\begin{align}
		|A_1|_{W^{0,\infty}([-\alpha-d,\alpha+d])}\leq& \sum_{\mb,\vb}c_{i,\mb,\vb}|\widetilde{\times}(a,b)-ab|_{W^{0,\infty}([-\alpha-D,\alpha+D])} \nonumber\\
		\leq& C_9N^{d}\delta,
		\label{eq.A1.0}
	\end{align}
	and for $k=1$
	\begin{align}
		&|A_1|_{W^{1,\infty}([-\alpha-d,\alpha+d])} \nonumber\\
		\leq& \sum_{\mb,\vb}c_{i,\mb,\vb}C'|\widetilde{\times}(a,b)-ab|_{W^{1,\infty}([-\alpha-D,\alpha+D])}|\widetilde{g}_{\mb,\vb}|_{W^{1,\infty}(\varphi_i(U_i))}\left|\widetilde{\mone}_i\circ\varphi_i^{-1}\right|_{W^{1,\infty}(\varphi_i(U_i))} \nonumber\\
		\leq& C_{10}N^{d+1}\delta/\Delta
		\label{eq.A1.1}
	\end{align}
	for some constants $C_9,C_{10},C'$ depending on $r,\alpha,d$, where we used Lemma \ref{lem.ATE.approx.error} and (\ref{eq.product.3}) in the last inequality.
	Combining (\ref{eq.A1.0}) and (\ref{eq.A1.1}) gives rise to
	\begin{align}
		\|A_1\|_{W^{k,\infty}([-\alpha-d,\alpha+d])}\leq C_{11}N^{d+k}\delta/\Delta
		\label{eq.A1}
	\end{align}
	for $k=0,1$ and a constant $C_{11}$ depending on $d,\alpha,r$.
	
	Before we derive upper bounds for $A_2$ and $A_3$, we define some sets which will be used in our following proof.
	
	Define the set 
	\begin{align*}
		\widetilde{\Omega}_{i,1}=\left\{\xb\in U_i: \min_{\widetilde{\xb}\in \partial U_i} \|\xb-\widetilde{\xb}\|_2\leq c\right\},
	\end{align*}
	where $c$ is the constant from Lemma \ref{lem.parunity}. Denote $\Omega_{i,1}=\varphi_i(\widetilde{\Omega}_{i,1})$. According to Lemma \ref{lem.parunity}, we have $f_i|_{\widetilde{\Omega}_{i,1}}=f_i\circ\phi_i|_{\Omega_{i,1}}=0$. Since $\varphi_i$ is a bijection, both $\Omega_{i,1}$ and $\widetilde{\Omega}_{i,1}$ have two disjoint boundaries. Denote the two boundaries of $\Omega_{i,1}$ by $\lambda_{i,1,1}$ and $\lambda_{i,1,2}$. We define the thickness of $\Omega_{i,1}$ as
	\begin{align*}
		\chi_{i,1}=\min_{\zb\in \lambda_{i,1,1},\ \widetilde{\zb}\in \lambda_{i,1,2}} \|\zb-\widetilde{\zb}\|_2.
	\end{align*}
	Since each $\varphi_i$ is a bijection, there exists 
	a constant $c_1$ depending on $c$ and the atlas such that $\chi_{i,1}\geq c_1$ for all $i$'s.
	Again since $\phi_i$ is a linear bijection, its inverse exists and is linear, and there exists a constant $c_2$ such that 
	\begin{align}
		\|\varphi_i^{-1}(\zb)-\varphi_i^{-1}(\widetilde{\zb})\|_2\geq c_2\|\zb-\widetilde{\zb}\|_2.
		\label{eq.lowLip}
	\end{align}
	
	We will choose $\theta$ and $\Delta$ small enough such that 
	\begin{align}
		\frac{8B^2D\theta}{c_2} \leq \frac{\Delta}{c_2r} \leq \frac{c_1}{2}.
		\label{eq.setDelta}
	\end{align}

	Define the region
	\begin{align}
		\Omega_{i,2}=\left\{\zb\in \varphi_i(U_i): \min_{\widetilde{\zb}\in \phi_i(\partial U_i)} \|\zb-\widetilde{\zb}\|_2\leq \frac{\Delta}{c_2r}\right\}.
		\label{eq.Omega2}
	\end{align}
	According to (\ref{eq.lowLip}), (\ref{eq.setDelta}) and the definition of $\Omega_{i,1}$, we have $\Omega_{i,2}\subset \Omega_{i,1}$. 
	For any $\zb\in \varphi_i(U_i)\backslash\Omega_{i,2}$, denote $\zb^*=\argmin_{\widetilde{\zb}\in \varphi_i(\partial U_i)} \|\zb-\widetilde{\zb}\|_2$. We have 
	$$
	\min_{\widetilde{\xb}\in \partial U_i}\|\varphi_i^{-1}(\zb)-\widetilde{\xb}\|_2\geq c_2\|\zb-\zb^*\|_2\geq \Delta/r. 
	$$
	Therefore 
	$$
	\|\varphi_i^{-1}(\zb)-\cbb_i\|_2^2\leq (r-\Delta/r)^2=r^2+\left(\frac{\Delta}{r}\right)^2-2\Delta\leq r^2-\Delta
	$$ 
	when $\Delta\leq r^2$ and 
	$$
	\widetilde{\mone}_i\circ\varphi^{-1}(\zb)=1, \quad \left.\frac{\partial \widetilde{\mone}_i\circ\varphi^{-1}}{z_j}\right|_{\varphi_i(U_i)\backslash \Omega_{i,2}}=0
	$$
	for $j=1,...,d$, where we used the notation $\zb=[z_1,...,z_d]^{\top}$.
	
	Note that each $\widetilde{g}_{\mb,\vb}$ and $\phi_{\mb}\zb^{\vb}$ is supported on $B_{1/N,\|\cdot\|_{\infty}}(\mb/N)$, a hyper cube with edge length $2/N$. 
We will choose $N$ large enough such that 
	$$
	\frac{2}{N}\leq \frac{\Delta}{4c_2r}\leq \frac{c_1}{8}.
	$$
	Such a choice of $N$ ensures that along any directions of $z_j$ for $j=1,...,d$, there are at least 2 hypercubes that entirely locate inside $\Omega_{i,2}$.
	Since any $\zb\in [0,1]^d$ is only covered by 2 hypercubes along each coordinate direction, we have 
	\begin{align}
		\{c_{i,\mb,\vb}: \mbox{ there exits }\zb\in \varphi_i(\partial U_i) \mbox{ such that } \zb\in B_{1/N,\|\cdot\|_{\infty}}(\mb/N)\}=0
		\label{eq.cis0}
	\end{align}
	and $\widetilde{f_i}\circ\varphi_i(\zb)=0$ for any $\zb\in \varphi_i(\partial U_i)$.
	See Figure \ref{fig.Omega} for an illustration.
	\begin{figure}[ht!]
		\centering
		\includegraphics[width=0.4\textwidth]{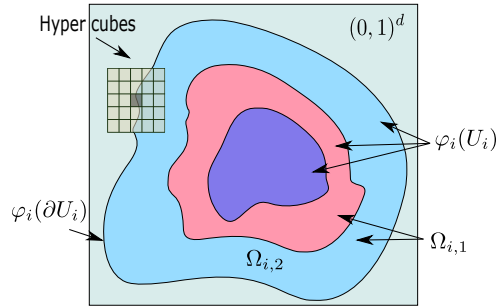}
		\caption{Illustration of the relations of $\Omega_{i,1},\ \Omega_{i,2}$ and $\varphi_i(U_i)$.}
		\label{fig.Omega}
	\end{figure}
	
	We have the following lemma on the bound of $\|A_2\|_{W^{k,\infty}(\varphi_i(U_i))}$ (see Appendix \ref{sec.proof.A2} for a proof):
		\begin{lemma}\label{lem.A2}
			Let $A_2$ be defined as in (\ref{eq.M.A2}). Assume $\Delta\leq r^2$. We have 
			\begin{align}
				\|A_2\|_{W^{1,\infty}(\varphi_i(U_i))}\leq C_{12}N\eta\Delta^{1-k}
				\label{eq.A2}
			\end{align}
		for $k=1,2$.
		\end{lemma}

	The term $A_3$ can be bounded using Lemma \ref{lem.net.error}:
	\begin{align}
		\|A_3\|_{W^{k,\infty}(\varphi_i(U_i))}=&\left|\sum_{\mb,\vb}c_{i,\mb,\vb}\left[\widetilde{g}_{\mb,\vb}(\zb)-\phi_{\mb}(\zb)\zb^{\vb}\right]\right|_{W^{k,\infty}(\varphi_i(U_i))} \nonumber\\
		\leq & \left|\sum_{\mb,\vb}c_{i,\mb,\vb}\left[\widetilde{g}_{\mb,\vb}(\zb)-\phi_{\mb}(\zb)\zb^{\vb}\right]\right|_{W^{k,\infty}([0,1]^d)} \nonumber\\
		\leq &C_{13}N^k\eta
		\label{eq.A3}
	\end{align}
for some constant $C_{13}$ depending on $d,\alpha,R$.
	Substituting (\ref{eq.A1}), (\ref{eq.A2}) and (\ref{eq.A3}) into (\ref{eq.M.net.err}) gives rise to
	\begin{align}
		&\|\widetilde{f}_i\circ\varphi_i^{-1}(\zb)-\widehat{f}_i(\zb)\|_{W^{k,\infty}(\varphi_i(U_i))} \leq   C_{11}N^{d+k}\delta/\Delta+ C_{12}N^k\eta +C_{13}N\eta\Delta^{1-k}.
		\label{eq.M.net.err.1}
	\end{align}
	The second term in (\ref{eq.M.err}) can be bounded by Lemma \ref{lem.ATE.approx.error} as
	\begin{align}
		\|\widehat{f}_i(\zb)-f_i\circ\varphi_i^{-1}(\zb)\|_{W^{k,\infty}([0,1]^d)}\leq C_{14}N^{-(\alpha-k)}.
		\label{eq.M.taylor.err.1}
	\end{align}
	Substituting (\ref{eq.M.net.err.1}) and (\ref{eq.M.taylor.err.1}) into (\ref{eq.M.err}) gives rise to
	\begin{align*}
		&\|\widetilde{f}-f\|_{W^{k,\infty}(\cM)} \leq C_{\cM}C_{11}N^{d+k}\delta/\Delta+ C_{\cM}C_{12}N^k\eta +C_{\cM}C_{13}N\eta\Delta^{1-k} +  C_{\cM}C_{14}N^{-(\alpha-k)}.
	\end{align*}
	Setting 
	\begin{align*}
		\eta=N^{-\alpha},\ \Delta=8c_2rN^{-1},\ \delta=N^{-(\alpha+d+1)},\ \theta=(8B^2D)^{-1}\Delta,
	\end{align*}
	we have
	\begin{align}
		\|\widetilde{f}-f\|_{W^{k,\infty}(\cM)}\leq C_{15}N^{-(\alpha-k)}
		\label{eq.M.err.1}
	\end{align}
	for $k=0,1$ and a constant $C_{15}$ depending on $d,\alpha,\tau$ and the surface area of $\cM$.
	
	\noindent$\bullet$\textbf{Network size}
	We analyze the network size for each $\widetilde{f}_i$:
	\begin{itemize}
		\item $\widetilde{\mone}_i$: The chart dermination network is the composition of $\widetilde{d}_i$ and $\widetilde{\mone}_{\Delta}$, where $\widetilde{d}_i$ has $O(\log \frac{1}{\theta})=O(\log N+\log D)$ layers and $O(D)$ width, $\widetilde{\mone}_{\Delta}$ has $O(\log \frac{1}{\delta})+D=O(\log N)+D$ layers and $O(1)$ width. In both subnetworks, all parameters are of $O(1)$. By Lemma \ref{lem.cnn.composition}, the chart dermination network has $O(\log N+\log D)+D$ layers, $O(D)$ width and all weight parameters are of $O(1)$.
		\item $\widetilde{\times}$: The multiplication network has $O(\log \frac{1}{\delta})=O(\log N)$ layers, $O(1)$ width. All weight parameters are bounded by $2(\alpha+d+1)$.
		\item $\varphi_i$: the projection $\varphi_i$ can be realized by a single layer with width $d$. All parameters are of $O(1)$.
		\item $\widetilde{g}_{i,\mb,\vb}$: By Lemma \ref{lem.product}, each $\widetilde{g}_{i,\mb,\vb}$ has $O(\log N)$ layers and $O(d)$ width. All parameters are of $O(N)$.
		\item $c_{i,\mb,\vb}$: By Lemma \ref{lem.ATE.approx.error} with $p=\infty$, each $c_{i,\mb,\vb}$ is of $O(1)$.
	\end{itemize}
	By Lemma \ref{lem.cnn.composition}, each $c_{i,\mb,\vb}(\widetilde{g}_{\mb,\vb}\circ\varphi_i(\xb))\widetilde{\times} \widetilde{\mone}_i(\xb)$ is a CNN with $O(\log N+\log D)+D$ layers, $O(D)$ width and all parameters of $O(N)$. According to (\ref{eq.tildef.1}), $\widetilde{f}$ can be written as a sum of $C_{\cM}N^d(d+1)^{\alpha}$ CNNs
	\begin{align}
		\widetilde{f}=\sum_{i=1}^{C_{\cM}}\sum_{\mb,\vb} c_{i,\mb,\vb}(\widetilde{g}_{\mb,\vb}\circ\varphi_i(\xb))\widetilde{\times} \widetilde{\mone}_i(\xb).
		\label{eq.tildef.2}
	\end{align}
	By Lemma \ref{lem.CNN.adap}, for any $\widetilde{M},\widetilde{J}$ satisfying $\widetilde{M}\widetilde{J}=O(N^d)$, there exists a CNN architecture $\cF^{\rm CNN}(L,J,K,\kappa,\kappa)$ that gives rise to $\{g_i\}_{i=1}^n$ with 
	\begin{align*}
		\widetilde{f}=\sum_{i=1}^{\widetilde{M}} g_i
	\end{align*}
	and 
	\begin{align*}
		L=O\left(\log N+\log D\right)+D,\ J=O(D\widetilde{J}),\ \kappa=O(N).
	\end{align*}

By Lemma \ref{lem.cnn.convresnet}, there exits a ConvResNet architecture  $\cC(M,L,J,K,\kappa_1,\kappa_2)$ with
\begin{align}
	L=O(\log N)+D, J=O(D\widetilde{J}), \kappa_1=\kappa_2=O(N), M=O(\widetilde{M})
	\label{eq.M.net.para.1}
\end{align}
and $\widetilde{J},\widetilde{M}$ satisfying
\begin{align}
	\widetilde{M}\widetilde{J}=O(N^d),
	\label{eq.M.net.para.2}
\end{align}
that yields a ConvResNet realizing $\widetilde{f}$.
Setting $N=O((\widetilde{M}\widetilde{J})^{1/d})$ in (\ref{eq.M.err.1}) and (\ref{eq.M.net.para.1}) gives rise to
\begin{align}
	\|\widetilde{f}-f\|_{W^{k,\infty}(\cM)}\leq C_{15}(\widetilde{M}\widetilde{J})^{-\frac{\alpha-k}{d}}
	\label{eq.M.err.2}
\end{align}
and the network size
\begin{align*}
	L=O\left(\log (\widetilde{M}\widetilde{J})+\log D\right)+D,\ J=O(D\widetilde{J}),\ \kappa=O((\widetilde{M}\widetilde{J})^{1/d}).
\end{align*}
\end{proof}

\section{Definitions, Lemmas and their proofs used in Section \ref{sec.proof.D}}\label{sec.support}
\subsection{Existing lemmas on CNNs}
Lemma \ref{lem.cnnRealization} shows that any MLP can be realized by a CNN.
\begin{lemma}[Theorem 1 in \citet{oono2019approximation}]
	\label{lem.cnnRealization}
	Let $D$ be the dimension of the input. Let $L,J$ be positive integers and $\kappa>0$. For any $2\leq K'\leq D$, any MLP architectures $\cF^{\rm MLP}(L,J,\kappa)$ can be realized by a CNN architecture $\cF^{\rm CNN}(L',J',K',\kappa_1',\kappa_2')$ with
	$$
	L'=L+D, J'=4J, \kappa'_1=\kappa'_2=\kappa.
	$$
	Specifically, any $\bar{f}^{\rm MLP}\in \cF^{\rm MLP}(L,J,\kappa)$ can be realized by a CNN $\bar{f}^{\rm CNN}\in \cF^{\rm CNN}(L',J',K',\kappa_1',\kappa_2')$.
	Furthermore, the weight matrix in the fully connected layer of $\bar{f}^{\rm CNN}$ has nonzero entries only in the first row.
\end{lemma}
Lemma \ref{lem.cnn.composition} shows that the composition of two CNNs can be realized by a CNN.
\begin{lemma}[Lemma 13 in \citet{liu2021besov}]\label{lem.cnn.composition}
	Let $\cF_1^{\rm CNN}(L_1,J_1,K_1,\kappa_1,\kappa_1)$ be a CNN architecture from $\RR^D\rightarrow \RR$ and $\cF_2^{\rm CNN}(L_2,J_2,K_2,\kappa_2,\kappa_2)$ be a CNN architecture from $\RR\rightarrow\RR$. Assume the weight matrix in the fully connected layer of $\cF_1^{\rm CNN}(L_1,J_1,K_1,\kappa_1,\kappa_1)$ and $\cF_2^{\rm CNN}(L_2,J_2,K_2,\kappa_2,\kappa_2)$ has nonzero entries only in the first row. Then there exists a CNN architecture $\cF^{\rm CNN}(L,J,K,\kappa,\kappa)$ from $\RR^D\rightarrow \RR$ with
	\begin{align*}
		L=L_1+L_2, \ J=\max(J_1,J_2),\ K=\max(K_1,K_2), \kappa=\max(\kappa_1,\kappa_2)
	\end{align*}
	such that for any $f_1\in \cF^{\rm CNN}(L_1,J_1,K_1,\kappa_1,\kappa_1)$ and $f_2\in \cF^{\rm CNN}(L_2,J_2,K_2,\kappa_2,\kappa_2)$, there exists $f\in \cF^{\rm CNN}(L,J,K,\kappa,\kappa)$ such that $f(\xb)=f_2\circ f_1(\xb)$.
	Furthermore, the weight matrix in the fully connected layer of $\cF^{\rm CNN}(L,J,K,\kappa,\kappa)$ has nonzero entries only in the first row.
\end{lemma}

\subsection{Interpolation spaces}\label{sec.interpolation}
\begin{definition}[Interpolation spaces]
	Let $(B_0,B_1)$ be an interpolation couple. For any $u\in B_1$, define 
	$$
	K(t,u,B_0,B_1)=\inf_{v\in B_1} \left(\|u-v\|_{B_0}+t\|v\|_{B_1}\right) 
	$$
	and the norm
	\begin{align*}
		\|u\|_{(B_0,B_1)_{\theta,p}}=\begin{cases}
			\left(\int_{0}^{\infty}  t^{-\theta p}K(t,u,B_0,B_1)^p \frac{dt}{t} \right)^{1/p}, &\mbox{ for } 1\leq p < \infty,\\
			\sup_{0<t<\infty}  t^{-\theta} K(t,u,B_0,B_1), &\mbox{ for } p=\infty.
		\end{cases}
	\end{align*}
	Then the interpolation space $(B_0,B_1)_{\theta,p}$  is defined by
	\begin{align*}
		(B_0,B_1)_{\theta,p}=\left\{ u\in B_0: \|u\|_{(B_0,B_1)_{\theta,p}} < \infty\right\}.
	\end{align*}
\end{definition}

The following lemma shows that the fractional Sobolev space is an interpolation space:
\begin{lemma}[Theorem 14.2.3 of \citet{brenner2008mathematical}]
	Let $\Omega\in \RR^D$ be an Lipschitz domain. Then for any $0<s<1$ and $1\leq p\leq \infty$, we have
	\begin{align*}
		W^{s,p}(\Omega)=(L^p(\Omega),W^{1,p}(\Omega))_{s,p}.
	\end{align*}
\end{lemma}
The following lemma shows that the norm of the interpolation space of $(B_0,B_1)_{\theta,p}$ can be bounded using $\|\cdot\|_{B_0}$ and $\|\cdot\|_{B_1}$:
\begin{lemma}\label{lem.interpolation.norm}
	Let $(B_0,B_1)$ be an interpolation couple. Moreover, let $0<\theta<1$ and $1\leq p \leq \infty$. Then there exists a constant $C$ depending on $\theta$ and $p$ such that for all $u\in B_1$, we have
	\begin{align*}
		\|u\|_{B_{\theta,p}}\leq C\|u\|_{B_0}^{1-\theta}\|u\|_{B_1}^{\theta}.
	\end{align*}
	In particular, when $p=\infty$, we have $C=1$.
\end{lemma}

\subsection{Proof of Lemma \ref{lem.psi}} \label{sec.psi}
\begin{proof}[Proof of Lemma \ref{lem.psi}]
	Note that $\psi(x)$ can be realized by a two-layer MLP
	\begin{align*}
		\psi(x)&=\ReLU(A_2\cdot \ReLU\left(A_1x+\bb_1 \right))
	\end{align*} 
	with 
	$$
	A_1=\begin{bmatrix}
		1 \\1 \\1 \\1
	\end{bmatrix},\ \bb_1= 	\begin{bmatrix}
		2\\ 1 \\-1\\-2
	\end{bmatrix},\ A_2=\begin{bmatrix}
		1 & -1 & -1 &1
	\end{bmatrix}.
	$$
	According to Lemma \ref{lem.cnnRealization}, for any $2\leq K$, such an MLP can be realized by a CNN in $\cF^{\rm CNN}(2,16,2,2,2)$. According to the expression of the right-hand-side of (\ref{eq.psi.CNN}), we have $\widetilde{\psi}_{m,N}(x)\in \cF^{\rm CNN}(2,16,2,3N,3N)$.
	
	To prove (\ref{eq.psi.sobolev}), the case $k=0$ follows by the definition of $\psi$. For $k=1$, we have 
	\begin{align*}
		\frac{d \widetilde{\psi}_{m,N}(x)}{d x}=\frac{d \psi\left(3N\left(x_k-\frac{m}{N}\right)\right)}{d x}=3N.
	\end{align*}
\end{proof}

\subsection{Proof of Lemma \ref{lem.product}}\label{sec.product}
\begin{proof}[Proof of Lemma \ref{lem.product}]
	For any given $\mb$ and $\vb$, $\phi_{\mb}\xb^{\vb}$ is a product of at most $\alpha+D$ quantities each of which can be realized by a CNN. The following lemma shows that the multiplication operator $\times$ can be well approximated by a CNN (see a proof in Appendix \ref{sec.multiplication.CNN}):
	
	\begin{lemma}\label{lem.multiplication.CNN}
		For any $0<\eta<1/2$, $x,y\in[-B,B]$, and $K\geq 2$, there is a CNN architecture $\cF^{\rm CNN}(L,J,K,\kappa,\kappa)$ that yields a CNN, denoted by $\ttimes(\cdot,\cdot)$, such that
		$$
		\|\ttimes(x,y)-xy\|_{W^{1,\infty}[-B,B]^2}<\eta,\ \ttimes(x,0)=\ttimes(y,0)=0.
		$$
		Such a network has 
		$$
		L=O\left(\log \frac{1}{\eta}\right),\ J=24,\ \kappa=1.
		$$
		Furthermore, one has
		$$
		\|\ttimes(x,y)\|_{W^{1,\infty}((-B,B)^2)}\leq CB
		$$
		for some absolute constant $C$.
		
	\end{lemma}
	
	For simplicity, we denote $\psi_{m_k}(x)=\psi\left(3N\left(x_k-\frac{m_k}{N}\right)\right)$ for $k=1,...,D$.
	Then we construct $\widetilde{g}_{\mb,\vb}(\xb)$ as
	\begin{align*}
		\widetilde{g}_{\mb,\vb}(\xb)=\ttimes(\ttimes(...\ttimes(\ttimes(\widetilde{p}_{\vb}(\xb),\psi_{m_1}(x_1)),\psi_{m_2}(x_2)),...,),\psi_{m_D}(x_D)),
	\end{align*}
	where $\widetilde{p}_{\vb}(\xb)$ is the network approximation of $\xb^{\vb}$ defined by
	\begin{align*}
		\widetilde{p}_{\vb}(\xb)=\ttimes(...\ttimes(x_1,x_1),...,x_D).
	\end{align*}

	The structure of $\widetilde{g}_{\mb,\vb}$ is visualized in Figure \ref{fig.multiplication}. Here  $\widetilde{g}_{\mb,\vb}$ consists of no  more than $\alpha+D-1$ compositions of $\ttimes$ and $2D$ additional channels. These additional channels are used to pass the information $\xb_+$ and $\xb_-$.
	

	\begin{figure}[ht!]
		\centering
		\includegraphics[width=0.6\textwidth]{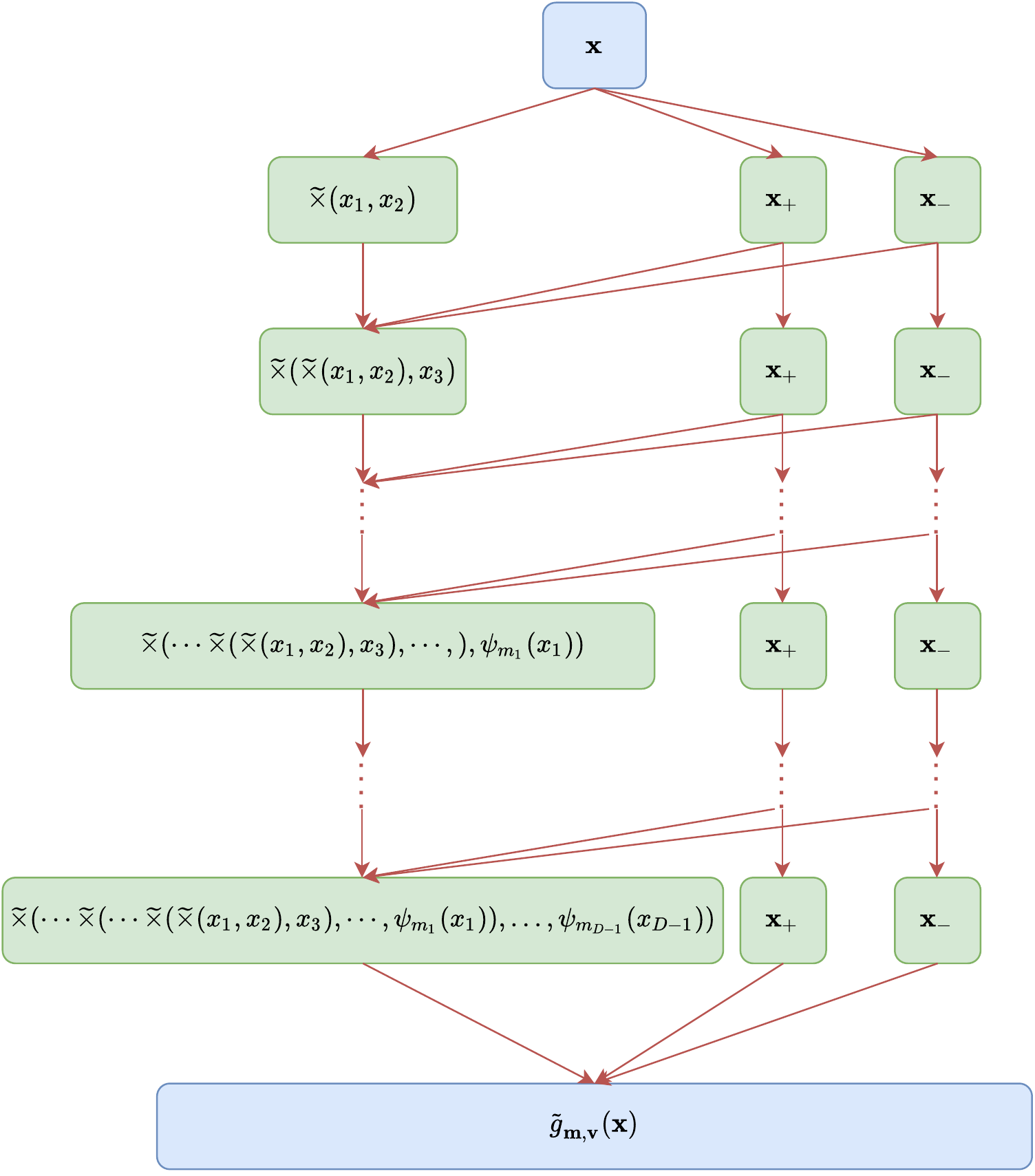}
		\caption{Illustration of $\widetilde{g}_{\mb,\vb}$.}
		\label{fig.multiplication}
	\end{figure}
	
	By applying Lemma \ref{lem.cnn.composition} $\alpha+D-2$ times, we have $\widetilde{g}_{\mb,\vb}\in \cF^{\rm CNN}(L,J,K,\kappa,\kappa)$ with 
	\begin{align*}
		L=O\left(D\log \frac{1}{\varepsilon}\right),\ J=O(D),\ \kappa=3N.
	\end{align*}
	
	We next prove (\ref{eq.product}) and (\ref{eq.product0}). First note that we can express
	\begin{align*}
		&\phi_{\mb}\xb^{\vb}=g_n\equiv\prod_{i=1}^n h_i(\xb),\\
		&\widetilde{g}_{\mb,\vb}(\xb)=\widetilde{g}_n\equiv \ttimes(\ttimes(\cdots\ttimes(h_1(\xb),h_2(\xb)),...),h_n(\xb))
	\end{align*}
	for some $n\leq \alpha+D$, where each $h_i$ can be realized by one layer and satisfies
	$$
	\|h_i(\xb)\|_{W^{k,\infty}(0,1)}\leq (3N)^k.
	$$
	
	To prove (\ref{eq.product}) and (\ref{eq.product0}), it is enough to show 
	\begin{align}
		&\|\widetilde{g}_n(\xb)-g_n(\xb)\|_{W^{k,\infty}((0,1)^D}\leq n^{1-k}c_n^kN^k\varepsilon
		\label{eq.product.1}\\
		&\widetilde{g}_n(\xb)=0 \mbox{ if } g_n(\xb)=0,
		\label{eq.product.2}\\
		&|\widetilde{g}_n(\xb)|_{W^{1,\infty}((0,1)^D)}\leq C_{16}N^k
		\label{eq.product.3}
	\end{align}
	for any $1\leq n\leq \alpha+D-1$, where $\{c_n\}_{n=1}^{\alpha+D-1}$ and $C_{16}$ are constants depending on $D$ and $\alpha$.
	
	For $n=1$, we have 
	\begin{align*}
		|\widetilde{g}_{n}-g_n|_{W^{k,\infty}((0,1)^D)}= |\ttimes(h_1,1)-h_1|_{W^{k,\infty}((0,1)^D)}.
	\end{align*}
	By Lemma \ref{lem.multiplication.CNN} with $B=\alpha+D+1$, we have for $k=0$, 
	\begin{align*}
		|\ttimes(h_1,1)-h_1|_{W^{0,\infty}((0,1)^D)}\leq \varepsilon.
	\end{align*}
	For $k=1$, by Lemma \ref{lem.multiplication.CNN}, we deduce
	\begin{align*}
		|\ttimes(h_1,1)-h_1|_{W^{1,\infty}((0,1)^D)}\leq C'|\ttimes(x,y)-x\cdot  y|_{W^{1,\infty}([0,1]^2)}|h_1|_{W^{1,\infty}([0,1]^2)}\leq 3C'N\varepsilon,
	\end{align*}
	where $C'$ is a constant depending on $D$. We set $c_1=3C'$. Furthermore, 
	\begin{align*}
		|\widetilde{g}_1(\xb)|_{W^{1,\infty}((0,1)^D)}=|\ttimes(h_1,1)|_{W^{1,\infty}((0,1)^D)}\leq C_4|\ttimes(x,y)-x\cdot y|_{W^{1,\infty}((0,1)^2)}|h_1|_{W^{1,\infty}((0,1)^2)}\leq C_5N,
	\end{align*}
	where $C_{17},C_{18}$ are constants depending on $D,\alpha$.
	
	Therefore, the inequalities (\ref{eq.product.1}) and (\ref{eq.product.3}) hold for $n=1$. 
	
	For (\ref{eq.product.2}), if $g_1(\xb)=0$, then $h_1(\xb)=0$. By Lemma \ref{lem.multiplication.CNN}, $\widetilde{g}_1(\xb)=0$. 
	
	Assume (\ref{eq.product.1})--(\ref{eq.product.3}) hold for any $1\leq n\leq t$ for some integer $t$ satisfying $1\leq t\leq \alpha+D-2$, i.e., for any $1\leq n\leq t$, we have
	\begin{align}
		&|\widetilde{g}_n-g_n|_{W^{k,\infty}((0,1)^D)}\leq n^{1-k}c_n^kN^k\varepsilon, \label{eq.gt.W1}\\
		&\widetilde{g}_n=0 \mbox{ if } g_n=0,\\
		&|\widetilde{g}_n|_{W^{1,\infty}((0,1)^D)}\leq C_{19}N. \label{eq.gt.mag}
	\end{align}
	
	We also deduce that 
	\begin{align}
		|\widetilde{g}_{t}|_{W^{0,\infty}((0,1)^D)}=|\widetilde{g}_{t}-g_{t}|_{W^{0,\infty}((0,1)^D)}+|g_{t}|_{W^{0,\infty}((0,1)^D)}\leq t\varepsilon +1 \leq t+1.
		\label{eq.gt.W0}
	\end{align}
	
	For $n=t+1$, we have
	\begin{align}
		|\widetilde{g}_{t+1}-g_{t+1}|_{W^{k,\infty}((0,1)^D)}&= |\ttimes(\widetilde{g}_t,h_{t+1})-g_t\cdot h_{t+1}|_{W^{k,\infty}((0,1)^D)} \nonumber\\
		&\leq |\ttimes(\widetilde{g}_t,h_{t+1})-\widetilde{g}_t\cdot h_{t+1}|_{W^{k,\infty}((0,1)^D)}+|\widetilde{g}_t\cdot h_{t+1}-g_t\cdot h_{t+1}|_{W^{k,\infty}((0,1)^D)}.
		\label{eq.product.5}
	\end{align}
	
	Consider the first term in (\ref{eq.product.5}). For $k=0$, we have
	\begin{align}
		|\ttimes(\widetilde{g}_t,h_{t+1})-\widetilde{g}_t\cdot h_{t+1}|_{W^{0,\infty}((0,1)^D)}\leq |\ttimes(x,y)-x\cdot y|_{W^{0,\infty}([-t-1,t+1]^2)}\leq \varepsilon.
		\label{eq.product.1.1}
	\end{align}
	For $k=1$, we have
	\begin{align}
		&|\ttimes(\widetilde{g}_t,h_{t+1})-\widetilde{g}_t\cdot h_{t+1}|_{W^{1,\infty}((0,1)^D)} \nonumber\\
		\leq &C'|\ttimes(x,y)-x\cdot y|_{W^{1,\infty}([-t-1,t+1]^2)}|\widetilde{g}_t|_{W^{1,\infty}([-t-1,t+1]^2)}
		\leq 3C'c_t N\varepsilon,
		\label{eq.product.1.2}
	\end{align}
	where (\ref{eq.gt.W1}) with $k=1$ is used in the last inequality, $C'$ is a constant depending on $D$.
	
	For the second term in (\ref{eq.product.5}), we first consider $k=0$:
	\begin{align}
		|\widetilde{g}_t\cdot h_{t+1}-g_t\cdot h_{t+1}|_{W^{0,\infty}((0,1)^D)} \leq |h_{t+1}|_{\infty}|\widetilde{g}_t-g_t|_{\infty}\leq t\varepsilon,
		\label{eq.product.1.3}
	\end{align}
	where (\ref{eq.gt.W1}) with $k=0$ is used.
	
	For $k=1$, we have
	\begin{align}
		&|\widetilde{g}_t\cdot h_{t+1}-g_t\cdot h_{t+1}|_{W^{1,\infty}((0,1)^D)} \nonumber\\
		= & |h_{t+1}(\widetilde{g}_t-g_t)|_{W^{1,\infty}((0,1)^D)} \nonumber\\
		\leq & C_{20}|h_{t+1}|_{W^{1,\infty}((0,1)^D)}\|\widetilde{g}_t-g_t\|_{\infty} + C_{20}\|h_{t+1}\|_{\infty}|\widetilde{g}_t-g_t|_{W^{1,\infty}((0,1)^D)} \nonumber\\
		\leq & 3C_{20}Nt\varepsilon+C_{20}c_t N\varepsilon\leq C_{21}N\varepsilon,
		\label{eq.product.1.4}
	\end{align}
	where $C_{20},C_{21}$ are constants depending on $D$ and $\alpha$. In (\ref{eq.product.1.4}), (\ref{eq.gt.W1}) with $k=0$ and $k=1$ are used in the second inequality.
	
	Combining (\ref{eq.product.1.1})--(\ref{eq.product.1.4}) and setting $c_{t+1}=3C_{20}c_t+ C_{21}$ gives rise to
	\begin{align*}
		|\widetilde{g}_{t+1}-g_{t+1}|_{W^{k,\infty}((0,1)^D)}\leq (t+1)^{1-k}c_{t+1}^kN^k\varepsilon.
	\end{align*}
	Therefore, (\ref{eq.product.1}) holds for $n=t+1$.
	
	To prove (\ref{eq.product.2}), note that if $g_{t+1}=0$, then either $h_{t+1}=0$, or $g_t=0$. By our induction assumption, when $g_t=0$, we have $\widetilde{g}_t=0$. Since $\widetilde{g}_{t+1}=\ttimes(\widetilde{g}_t,h_{t+1})$, by Lemma \ref{lem.product}, we have $\widetilde{g}_{t+1}=0$ and (\ref{eq.product.2}) holds for $n=t+1$.
	
	For (\ref{eq.product.3}), we deduce
	\begin{align*}
		&|\widetilde{g}_{t+1}(\xb)|_{W^{1,\infty}((0,1)^D)} \nonumber\\
		=&|\ttimes(\widetilde{g}_t,h_{t+1})|_{W^{1,\infty}((0,1)^D)} \nonumber\\
		\leq &C'|\ttimes(x,y)-x\cdot y|_{W^{1,\infty}((-t-1,t+1)^2)}\max\left\{|\widetilde{g}_t|_{W^{1,\infty}((0,1)^2)},|h_{t+1}|_{W^{1,\infty}((0,1)^2)}\right\}\nonumber\\
		\leq &C_{22}N,
	\end{align*}
	where $C_{22}$ is a constant depending on $D$ and $\alpha$.

	Therefore, (\ref{eq.product.1})--(\ref{eq.product.3}) hold for $n=t+1$. By mathematical induction, (\ref{eq.product.1})--(\ref{eq.product.3}) hold for any $1\leq n\leq D+\alpha+1$, and (\ref{eq.product}) and (\ref{eq.product0}) are proved.
\end{proof}

\subsection{Proof of Lemma \ref{lem.multiplication.CNN}}\label{sec.multiplication.CNN}
\begin{proof}[Proof of Lemma \ref{lem.multiplication.CNN}]
	The proof of Lemma \ref{lem.multiplication.CNN} is based on the following lemma.
	\begin{lemma}[Proposition C.2 in \citet{guhring2020error}]
		For any $0<\eta<1/2$, $x,y\in[-B,B]$. There is an MLP, denoted by $\ttimes(\cdot,\cdot)$, such that
		$$
		\|\ttimes(x,y)-xy\|_{W^{1,\infty}[-B,B]^2}<\eta,\ \ttimes(x,0)=\ttimes(y,0)=0.
		$$
		Such a network has $O\left(\log \frac{1}{\eta}\right)$ layers and parameters. The width of each layer is bounded by 6 and all parameters are bounded by $2$. Furthermore, we have 
		$$
		\|\ttimes(x,y)\|_{W^{1,\infty}((-B,B)^2)}\leq CM,
		$$
		for some absolute constant $C$. 
		\label{lem.multiplication}
	\end{lemma}
	Combing Lemma \ref{lem.multiplication} and \ref{lem.cnnRealization}, for any $\varepsilon>0$, $K\geq 2$, there exits a CNN $\ttimes\in \cF^{\rm CNN}(L,J,K,\kappa,\kappa)$ such that for any $|x|\leq B,|y|\leq B$, we have
	\begin{align*}
		&|\ttimes(x,y)-xy|<\varepsilon,\ \ttimes(x,0)=\ttimes(y,0)=0,\\
		&\|\ttimes(x,y)\|_{W^{1,\infty}((-B,B)^2)}\leq C_{23}B,
	\end{align*}
	where $C_{23}$ is an absolute constant.
	Such an architecture has
	$$
	L=O\left(\log \frac{1}{\varepsilon}\right),\ J=24,\ \kappa=1.
	$$
\end{proof}

\subsection{Proof of Lemma \ref{lem.net.error}}\label{sec.net.error}
\begin{proof}[Proof of Lemma \ref{lem.net.error}]
	Denote $\Omega_{\mb,N}=B_{\frac{1}{N},\|\cdot\|_{\infty}}\left(\frac{\mb}{N}\right)$. We have
	\begin{align}
		&\|\sum_{\mb\in \cS_N} \phi_{\mb}Q_{\mb/N}^{\alpha}f-\sum_{\mb\in \cS_N} \sum_{|\vb|\leq \alpha-1}c_{\mb,\vb}\widetilde{g}_{\mb,\vb}\|_{W^{k,p}((0,1)^D)} \nonumber\\
		=&\left\|\sum_{\mb\in \cS_N} \sum_{|\vb|\leq \alpha-1}c_{\mb,\vb}\phi_{\mb}\xb^{\vb}-\sum_{\mb\in \cS_N} \sum_{|\vb|\leq \alpha-1}c_{\mb,\vb}\widetilde{g}_{\mb,\vb}\right\|_{W^{k,p}((0,1)^D)}^p \nonumber\\
		=&\left\|\sum_{\mb\in \cS_N} \sum_{|\vb|\leq \alpha-1}c_{\mb,\vb}\left(\phi_{\mb}\xb^{\vb}-\widetilde{g}_{\mb,\vb}\right)\right\|_{W^{k,p}((0,1)^D)}^p \nonumber\\
		\leq & \sum_{\widetilde{\mb}\in \cS_N} \left\|\sum_{\mb\in \cS_N} \sum_{|\vb|\leq \alpha-1}c_{\mb,\vb}\left(\phi_{\mb}\xb^{\vb}-\widetilde{g}_{\mb,\vb}\right)\right\|_{W^{k,p}(\Omega_{\widetilde{\mb},N}\cap (0,1)^D)}^p,
		\label{eq.net.error}
	\end{align}
	where the first equality follows from (\ref{eq.f.aveTaylor}), the last inequality holds since $(0,1)^D\subset \cup_{\widetilde{\mb}\in \cS_N} \Omega_{\widetilde{\mb},N}$.
	
	For each $\widetilde{\mb}$, we have
	\begin{align}
		&\left\|\sum_{\mb\in \cS_N} \sum_{|\vb|\leq \alpha-1}c_{\mb,\vb}\left(\phi_{\mb}\xb^{\vb}-\widetilde{g}_{\mb,\vb}\right)\right\|_{W^{k,p}(\Omega_{\widetilde{\mb},N}\cap (0,1)^D)} \nonumber\\
		\leq & \sum_{\mb\in \cS_N} \sum_{|\vb|\leq \alpha-1}|c_{\mb,\vb}|\left\|\phi_{\mb}\xb^{\vb}-\widetilde{g}_{\mb,\vb}\right\|_{W^{k,p}(\Omega_{\widetilde{\mb},N}\cap (0,1)^D)} \nonumber\\
		\leq & C_{24}N^{d/p}\sum_{\mb\in \cS_N} \sum_{|\vb|\leq \alpha-1}\|\bar{f}\|_{W^{\alpha-1,p}(\Omega_{\mb,N})}\left\|\phi_{\mb}\xb^{\vb}-\widetilde{g}_{\mb,\vb}\right\|_{W^{k,p}(\Omega_{\widetilde{\mb},N}\cap (0,1)^D)},
		\label{eq.net.error.decom}
	\end{align}
	where $C_{21}$ is the constant in Lemma \ref{lem.aveTaylor.err}, $\bar{f}$ is the extension of $f$ to $\RR^D$ from \citet[Theorem VI.3.1.5]{stein1970singular}, which satisfies
	\begin{align}
		\|\bar{f}\|_{W^{\alpha,p}(\RR^D)}\leq C_{25} \|f\|_{W^{\alpha,p}((0,1)^D)}
		\label{eq.extension}
	\end{align}
	for some constant $C_{25}$ depending on $D,p,\alpha$.
	
	We next derive an upper bound of the summand of (\ref{eq.net.error.decom}). We first deduce that
	\begin{align}
		&\left\|\phi_{\mb}\xb^{\vb}-\widetilde{g}_{\mb,\vb}\right\|_{W^{k,p}(\Omega_{\widetilde{\mb},N}\cap (0,1)^D)} \nonumber\\
		\leq &\left|\Omega_{\widetilde{\mb},N}\cap (0,1)^D\right|^{1/p} (D+1)^{1/p}\left\|\phi_{\mb}\xb^{\vb}-\widetilde{g}_{\mb,\vb}\right\|_{W^{k,\infty}(\Omega_{\widetilde{\mb},N}\cap (0,1)^D)} \nonumber\\
		\leq &C_{26}\left(\frac{1}{N}\right)^{d/p}\left\|\phi_{\mb}\xb^{\vb}-\widetilde{g}_{\mb,\vb}\right\|_{W^{k,\infty}(\Omega_{\widetilde{\mb},N}\cap (0,1)^D)} \nonumber\\
		\leq & C_{27}\left(\frac{1}{N}\right)^{d/p}N^k\eta,
		\label{eq.net.error.decom.1}
	\end{align}
	where $\left|\Omega_{\widetilde{\mb},N}\cap (0,1)^D\right|$ denotes the volume of $\Omega_{\widetilde{\mb},N}\cap (0,1)^D$, $C_{26},C_{27}$ are constants depending on $D,\alpha$ and $p$. We used Lemma \ref{lem.product} in the last inequality. Substituting (\ref{eq.net.error.decom.1}) into (\ref{eq.net.error.decom}) gives rise to
	\begin{align}
		&\left\|\sum_{\mb\in \cS_N} \sum_{|\vb|\leq \alpha-1}c_{\mb,\vb}\left(\phi_{\mb}\xb^{\vb}-\widetilde{g}_{\mb,\vb}\right)\right\|_{W^{k,p}(\Omega_{\widetilde{\mb},N}\cap (0,1)^D)} \nonumber\\
		= &C_{24}\sum_{\substack{\mb\in \cS_N\\ \|\mb-\widetilde{\mb}\|_{\infty}\leq 1}} \sum_{|\vb|\leq \alpha-1}\|\bar{f}\|_{W^{\alpha-1,p}(\Omega_{\mb,N})}\left\|\phi_{\mb}\xb^{\vb}-\widetilde{g}_{\mb,\vb}\right\|_{W^{k,p}(\Omega_{\widetilde{\mb},N}\cap (0,1)^D)} \nonumber\\
		\leq &C_{24}C_{27}N^k\eta\sum_{\substack{\mb\in \cS_N\\ \|\mb-\widetilde{\mb}\|_{\infty}\leq 1}} \sum_{|\vb|\leq \alpha-1}\|\bar{f}\|_{W^{\alpha-1,p}(\Omega_{\mb,N})} \nonumber\\
		\leq &C_{28}N^k\eta\sum_{\substack{\mb\in \cS_N\\ \|\mb-\widetilde{\mb}\|_{\infty}\leq 1}} \|\bar{f}\|_{W^{\alpha-1,p}(\Omega_{\mb,N})},
		\label{eq.net.error.decom.2}
	\end{align}
	where $C_{28}=C_{24}C_{27}(D+1)^{\alpha-1}$. By H\"{o}lder's inequality, we have
	\begin{align}
		&\sum_{\substack{\mb\in \cS_N\\ \|\mb-\widetilde{\mb}\|_{\infty}\leq 1}} \|\bar{f}\|_{W^{\alpha-1,p}(\Omega_{\mb,N})} \nonumber\\
		=& \sum_{\substack{\mb\in \cS_N\\ \|\mb-\widetilde{\mb}\|_{\infty}\leq 1}} \|\bar{f}\|_{W^{\alpha-1,p}(\Omega_{\mb,N})}\cdot 1 \nonumber\\
		\leq & \left(\sum_{\substack{\mb\in \cS_N\\ \|\mb-\widetilde{\mb}\|_{\infty}\leq 1}} \|\bar{f}\|_{W^{\alpha-1,p}(\Omega_{\mb,N})}^p\right)^{\frac{1}{p}}\left(\sum_{\substack{\mb\in \cS_N\\ \|\mb-\widetilde{\mb}\|_{\infty}\leq 1}} 1^q\right)^{\frac{1}{q}} \nonumber\\
		\leq & 3^{\frac{D}{q}} \left(\sum_{\substack{\mb\in \cS_N\\ \|\mb-\widetilde{\mb}\|_{\infty}\leq 1}} \|\bar{f}\|_{W^{\alpha-1,p}(\Omega_{\mb,N})}^p\right)^{\frac{1}{p}},
		\label{eq.net.error.decom.3}
	\end{align}
	where $q=1/(1-1/p)$. Substituting (\ref{eq.net.error.decom.2}), (\ref{eq.net.error.decom.3}) into (\ref{eq.net.error}) gives rise to
	\begin{align*}
		&\|\sum_{\mb\in \cS_N} \phi_{\mb}Q_{\mb/N}^{\alpha}f-\sum_{\mb\in \cS_N} \sum_{|\vb|\leq \alpha-1}c_{\mb,\vb}\widetilde{g}_{\mb,\vb}\|_{W^{k,p}((0,1)^D)}^p \nonumber\\
		\leq & \left(C_{28} 3^{\frac{D}{q}}N^k\eta \right)^p  \left(\sum_{\widetilde{\mb}\in \cS_N}\sum_{\substack{\mb\in \cS_N\\ \|\mb-\widetilde{\mb}\|_{\infty}\leq 1}} \|\bar{f}\|_{W^{\alpha-1,p}(\Omega_{\mb,N})}^p\right) \nonumber\\
		\leq & \left(C_{28} 3^{\frac{D}{q}}N^k\eta \right)^p 3^D \left(\sum_{\widetilde{\mb}\in \cS_N} \|\bar{f}\|_{W^{\alpha-1,p}(\Omega_{\widetilde{\mb},N})}^p\right) \nonumber\\
		\leq & \left(C_{28} 3^{\frac{D}{q}}N^k\eta \right)^p 3^D 2^D \|\bar{f}\|_{W^{\alpha-1,p}(\cup_{\widetilde{\mb}\in \cS_N} \Omega_{\widetilde{\mb},N})}^p \nonumber\\
		\leq & C_{29}N^{kp}\eta^{p} \|f\|_{W^{\alpha-1,p}((0,1)^D)},
	\end{align*}
	where $C_{29}$ is a constant depending on $D,\alpha,p$. In the above, we used (\ref{eq.extension}) in the last inequality. Lemma \ref{lem.net.error} is proved for $s=0$ and $s=1$. For any $0<s<1$ and $1\leq p\leq \infty$, by Lemma \ref{lem.interpolation.norm}, we have
	\begin{align*}
		\left\|\sum_{\mb\in \cS_N} \phi_{\mb}Q_{\mb/N}^{\alpha}f-\sum_{\mb\in \cS_N} \sum_{|\vb|\leq \alpha-1}c_{\mb,\vb}\widetilde{g}_{\mb,\vb}\right\|_{W^{k,p}((0,1)^D)}^p\leq& C_{30}N^{kp}\eta^{p} \|f\|_{W^{\alpha-1,p}((0,1)^D)}\\
		\leq& C_{30}N^{sp}\eta^{p} \|f\|_{W^{\alpha,p}((0,1)^D)}
	\end{align*}
	for some constant $C_{30}$ depending on $D,\alpha,s,p$. The proof is finished.
\end{proof}


\subsection{Lemma \ref{lem.CNN.sumseries} and its proof}\label{sec.CNN.sumseries}
\begin{lemma}\label{lem.CNN.sumseries}
	Let $\{f_i\}_{i=1}^{n}$ be a set of CNNs with architecture $\cF^{\rm CNN}(L_0,J_0,K_0,\kappa_0,\kappa_0)$. Then there for any integer $1\leq w\leq n$, there exists a CNN architecture $\cF^{\rm CNN}(L_w,J_w,K_w,\kappa_w,\kappa_w)$ that gives rise to a CNN $g_w$ such that
	\begin{align*}
		g_w(\xb)= \sum_{i=1}^{w} f_i(\xb).
	\end{align*}
	Such an architecture has 
	\begin{align*}
		L=O(L_0), \ J=wJ_0, \ K=K_0, \ \kappa=\kappa_0.
	\end{align*}
	Furthermore, the fully connected layer of $f$ has nonzero elements only in the first row.
\end{lemma}
\begin{proof}[Proof of Lemma \ref{lem.CNN.sumseries}]
	The idea of the proof is similar to \citet[proof of Lemma 14]{liu2021besov}. Following the proof of \citet[Lemma 14]{liu2021besov}, we can show that there exist a set of filters $\cW$ and biases $\cB$ such that 
	\begin{align*}
		\Conv_{\cW,\cB}(\xb)=\begin{bmatrix}
			(f_1(\xb))_+ & (f_1(\xb))_- & (f_2(\xb))_+ & (f_2(\xb))_- & \cdots & (f_w(\xb))_+ & (f_w(\xb))_-\\
			\star & \star & \star & \star &\cdots & \star & \star
		\end{bmatrix},
	\end{align*}
	where $\Conv_{\cW,\cB}$ has depth bounded by $L_0$, number of channels bounded by $wj_0$ and all weight parameters bounded by $\kappa_0$.
	We write $g_w$ as
	\begin{align*}
		g_w=W_1\cdot \Conv_{\cW,\cB},
	\end{align*}
	where $W_1$ is given as
	\begin{align*}
		W_1=\begin{bmatrix}
			1 & -1 & 1 & -1 & \cdots & 1 & -1\\
			\mathbf{0} & \mathbf{0} & \mathbf{0} & \mathbf{0} &\cdots &\mathbf{0} & \mathbf{0}
		\end{bmatrix}.
	\end{align*}
	The proof is finished.
\end{proof}
\subsection{Proof of Lemma \ref{lem.CNN.adap} }\label{sec.CNN.adap}
\begin{proof}[Proof of Lemma \ref{lem.CNN.adap}]
	For any given $\widetilde{J}$, let $c$ be the smallest integer such that $\widetilde{J}\leq cJ_0$. Then we set $J=cJ_0$ and $n=\lceil n_0/c \rceil$. By Lemma \ref{lem.CNN.sumseries}, there exists a CNN architecture $\cF^{\rm CNN}(L,J,K,\kappa,\kappa)$ with
	\begin{align*}
		L=O(L_0), \ J=cJ_0, \ K=K_0, \ \kappa=\kappa_0.
	\end{align*}
	Such an architecture gives rise to CNNs $\{g_j\}_{j=1}^{\lceil n_0/c\rceil}$ such that 
	$$
	g_j=\sum_{i=c(j-1)+1}^{\min\{cj,n\}} f_i.
	$$
	The lemma is proved.
\end{proof}

\section{Proof of lemmas in Appendix \ref{sec.proof.M}}
\subsection{Proof of Lemma \ref{lem.parunity}} \label{sec.proof.parunity}
\begin{proof}[Proof of Lemma \ref{lem.parunity}]
	Following the construction in \textbf{Step 1} of the proof of Theorem \ref{thm.M}, for $\widetilde{r}=r/2<\tau/8$, there exists a collection of points atlas of $\cM$ denoted by $\{\widetilde{U}_i,\widetilde{\varphi_i}\}_{i=1}^{\widetilde{C}_{\cM}}$, where $\widetilde{U}_i=B_{\widetilde{r}}(\widetilde{\cbb}_i)$ for some $\widetilde{\cbb}_i\in \cM$, and $\widetilde{\varphi}_i$ is defined according to (\ref{eq.transformation}). By \citet[Chapter 2 Equation (1)]{conwaysphere}, the number of charts is bounded by
	$$
	\widetilde{C}_{\cM}\leq \left\lceil \frac{\mathrm{SA}(\cM)}{\widetilde{r}^d}T_d\right\rceil =\left\lceil \frac{\mathrm{SA}(\cM)}{{r/2}^d}T_d\right\rceil.
	$$
	The following lemma shows that for any locally finite cover of a smooth manifold, a $C^{\infty}$ partition of unity always exists:
	\begin{lemma}[Chapter 2 Theorem 15 of \cite{spivak1973comprehensive}]\label{lem.parunity.exist}
		Let $\{U_\alpha\}_{\alpha \in \cA}$ be a locally finite cover of a smooth manifold $\cM$. There is a $C^\infty$ partition of unity $\{\rho_\alpha\}_{\alpha=1}^\infty$ such that $\supp(\rho_\alpha) \subset U_\alpha$. 
	\end{lemma}
	Let $\{\rho_i\}_{i=1}^{\widetilde{C}_{\cM}}$ be the partition of unity in Lemma \ref{lem.parunity.exist} with respect to $\{\widetilde{U}_i\}_{i=1}^{C_{\cM}}$.
	
	We set $C_{\cM}=\widetilde{C}_{\cM}$ and define $U_i=B_{r}(\widetilde{\cbb}_i)$ and $\varphi_i$ according to (\ref{eq.transformation}). Since $\widetilde{r}<r$, $\widetilde{U}_i\subset U_i$, we have $\widetilde{U}_i\subset U_i$ and
	$$
	\cM\subseteq \bigcup_{i=1}^{\widetilde{C}_{\cM}} \widetilde{U}_i \subseteq \bigcup_{i=1}^{{C}_{\cM}} U_i.
	$$
	Therefore $\{U_i\}_{i=1}^{C_{\cM}}$ is an open cover of $\cM$ and $\{U_i,\varphi_i\}_{i=1}^{C_{\cM}}$ is an atlas of $\cM$. Since $\supp(\rho_i)\subseteq \widetilde{U_i}$, we have $\supp(\rho_i)\subset U_i$ and 
	\begin{align*}
		\inf_{\xb\in \supp(\rho_i),\ \widetilde{\xb}\in \partial U_i} \|\xb-\widetilde{\xb}\|_{2} \geq \inf_{\xb\in \widetilde{U}_i,\ \widetilde{\xb}\in \partial U_i} \|\xb-\widetilde{\xb}\|_{2}=r/2.
	\end{align*}
	The lemma is proved.
\end{proof}

\subsection{Proof of Lemma \ref{lem.A2}}\label{sec.proof.A2}
\begin{proof}[Proof of Lemma \ref{lem.A2}]
We deduce 
\begin{align*}
	|A_2|_{W^{k,\infty}(\varphi_i(U_i))}=&\left| \left(\sum_{\mb,\vb}c_{i,\mb,\vb}\widetilde{g}_{\mb,\vb}(\zb)\right)\times \left(\widetilde{\mone}_i\circ\varphi_i^{-1}(\zb)- {\mone}_i\circ\varphi_i^{-1}(\zb)\right)\right|_{W^{k,\infty}(\varphi_i(U_i))} \nonumber\\
	=&\left| \widetilde{f}_i\circ\varphi_i^{-1}(\zb)\times \left(\widetilde{\mone}_i\circ\varphi_i^{-1}(\zb)- {\mone}_i\circ\varphi_i^{-1}(\zb)\right)\right|_{W^{k,\infty}(\varphi_i(U_i))} \nonumber\\
	\leq&\left| \widetilde{f}_i\circ\varphi_i^{-1}(\zb)\times \left(\widetilde{\mone}_i\circ\varphi_i^{-1}(\zb)- {\mone}_i\circ\varphi_i^{-1}(\zb)\right)\right|_{W^{k,\infty}(\Omega_{i,2})} \nonumber \\
	&+ \left| \widetilde{f}_i\circ\varphi_i^{-1}(\zb)\times \left(\widetilde{\mone}_i\circ\varphi_i^{-1}(\zb)- {\mone}_i\circ\varphi_i^{-1}(\zb)\right)\right|_{W^{k,\infty}(\varphi_i(U_i)\backslash \Omega_{i,2})} \nonumber\\
	=& \left| \widetilde{f}_i\circ\varphi_i^{-1}(\zb)\times \left(\widetilde{\mone}_i\circ\varphi_i^{-1}(\zb)- {\mone}_i\circ\varphi_i^{-1}(\zb)\right)\right|_{W^{k,\infty}(\Omega_{i,2})}
\end{align*}
for $k=0,1$, where the last equality holds since 
$$
\widetilde{\mone}_i\circ\varphi_i^{-1}(\zb)= {\mone}_i\circ\varphi_i^{-1}(\zb)=1
$$
on $\varphi_i(U_i)\backslash \Omega_{i,2}$. 

According to (\ref{eq.cis0}), $\widetilde{f}_i\circ\varphi_i^{-1}(\zb)=\widehat{f}_i\circ\varphi_i^{-1}(\zb)=0$ for $\zb\in \varphi_i(\partial U_i)$. For any $\zb\in \Omega_{i,2}$, let 
$$
\zb^*=\argmin_{\widetilde{\zb}\in \varphi_i(\partial U_i)}\|\zb-\widetilde{\zb}\|_2.
$$
According to (\ref{eq.Omega2}), we have $\|\zb-\zb^*\|_2\leq \Delta/(c_2r)$.

By Lemma \ref{lem.net.error} with some small $\eta>0$ and for $s=k=0,1$, we have 
\begin{align}
	\|\widetilde{f}_i\circ\varphi_i^{-1}-\widehat{f}_i\|_{W^{k,\infty}([0,1]^d)}\leq C_{31}N^{k}\eta,
	\label{eq.A2.W0}
\end{align}
where $C_{31}$ is a constant depending on $d,\alpha,R$.
Since $\|\widehat{f}_i\|_{W^{1,\infty}(\Omega_{i,2})}=0$, we have $\max_j \left|\frac{\partial \widetilde{f_i}}{\partial z_j}\right|\leq C_{31}N\eta$ for any $\zb\in \Omega_{i,2}$. Therefore
\begin{align}
	|\widetilde{f}_i\circ\varphi_i^{-1}(\zb)|\leq &\widetilde{f}_i\circ\varphi_i^{-1}(\zb^*)+C_{31}N\eta\|\zb-\zb^*\|_2
	\leq \frac{C_{31}}{c_2r}N\eta\Delta
	\label{eq.fi.omega2}
\end{align}
for any $\zb\in \Omega_{i,2}$.

Using $|\widetilde{\mone}_i\circ\varphi_i^{-1}(\zb)- {\mone}_i\circ\varphi_i^{-1}(\zb)|\leq 1$, we bound $A_2$ as 
\begin{align}
	|A_2|_{W^{0,\infty}(\varphi_i(U_i))}=& \left| \widetilde{f}_i\circ\varphi_i^{-1}(\zb)\times \left(\widetilde{\mone}_i\circ\varphi_i^{-1}(\zb)- {\mone}_i\circ\varphi_i^{-1}(\zb)\right)\right|_{W^{0,\infty}(\Omega_{i,2})} \nonumber\\
	\leq & \left| \widetilde{f}_i\circ\varphi_i^{-1}(\zb)\right|_{W^{0,\infty}(\Omega_{i,2})}\times \left|\widetilde{\mone}_i\circ\varphi_i^{-1}(\zb)- {\mone}_i\circ\varphi_i^{-1}(\zb)\right|_{W^{0,\infty}(\Omega_{i,2})} \nonumber\\
	\leq & \frac{C_{11}}{c_2r}N\eta\Delta
	\label{eq.A2.0}
\end{align}
for $k=0$ and 
\begin{align}
	|A_2|_{W^{1,\infty}(\varphi_i(U_i))}=& \left| \widetilde{f}_i\circ\varphi_i^{-1}(\zb)\times \left(\widetilde{\mone}_i\circ\varphi_i^{-1}(\zb)- {\mone}_i\circ\varphi_i^{-1}(\zb)\right)\right|_{W^{1,\infty}(\Omega_{i,2})} \nonumber\\
	\leq & \left| \widetilde{f}_i\circ\varphi_i^{-1}(\zb)\right|_{W^{0,\infty}(\Omega_{i,2})}\times \left|\widetilde{\mone}_i\circ\varphi_i^{-1}(\zb)- {\mone}_i\circ\varphi_i^{-1}(\zb)\right|_{W^{1,\infty}(\Omega_{i,2})} \nonumber\\
	&+ \left| \widetilde{f}_i\circ\varphi_i^{-1}(\zb)\right|_{W^{1,\infty}(\Omega_{i,2})}\times \left|\widetilde{\mone}_i\circ\varphi_i^{-1}(\zb)- {\mone}_i\circ\varphi_i^{-1}(\zb)\right|_{W^{0,\infty}(\Omega_{i,2})} \nonumber\\
	\leq & \frac{C_{31}C_8}{c_2r}N\eta\Delta/\Delta+C_{11}N\eta \nonumber\\
	=&C_{32}N\eta
	\label{eq.A2.1}
\end{align}
for $k=1$, where $C_{12}$ is a constant depending on $\alpha,R,\tau$. In the first inequality of (\ref{eq.A2.1}), we used (\ref{eq.fi.omega2}), the inequality 
\begin{align*}
	\left|\widetilde{\mone}_i\circ\varphi_i^{-1}(\zb)- {\mone}_i\circ\varphi_i^{-1}(\zb)\right|_{W^{1,\infty}(\Omega_{i,2})}=\left|\widetilde{\mone}_i\circ\varphi_i^{-1}(\zb)\right|_{W^{1,\infty}(\Omega_{i,2})}\leq C_8/\Delta
\end{align*}
by (\ref{eq.indi.partial}) and the fact ${\mone}_i\circ\varphi_i^{-1}(\zb)=1$ for $\zb\in \Omega_{i,2}$, and the inequality 
\begin{align*}
	\left| \widetilde{f}_i\circ\varphi_i^{-1}(\zb)\right|_{W^{1,\infty}(\Omega_{i,2})}=\left| \widetilde{f}_i\circ\varphi_i^{-1}(\zb)-0\right|_{W^{1,\infty}(\Omega_{i,2})}=\|\widetilde{f}_i\circ\varphi_i^{-1}-f_i\circ\varphi_i^{-1}\|_{W^{1,\infty}(\Omega_{i,2})}\leq C_{31}N\eta
\end{align*}
by (\ref{eq.A2.W0}).

Combining (\ref{eq.A2.0}) and (\ref{eq.A2.1}) gives rise to
\begin{align}
	\|A_2\|_{W^{1,\infty}(\varphi_i(U_i))}\leq C_{32}N\eta\Delta^{1-k}
\end{align}
for $k=0,1$.
\end{proof}

\end{document}